\documentclass[accepted]{uai2026} 
                        
\usepackage[american]{babel}

\usepackage{natbib} 
\usepackage{mathtools} 
\usepackage{booktabs} 
\usepackage{tikz} 

\usepackage{amsmath,amsfonts,bm}

\def\eqref#1{equation~\ref{#1}}

\def\1{\bm{1}}

\def\vone{{\bm{1}}}

\def\vf{{\bm{f}}}

\DeclareMathAlphabet{\mathsfit}{\encodingdefault}{\sfdefault}{m}{sl}
\SetMathAlphabet{\mathsfit}{bold}{\encodingdefault}{\sfdefault}{bx}{n}

\def\gC{{\mathcal{C}}}
\def\gD{{\mathcal{D}}}

\def\gF{{\mathcal{F}}}

\def\gH{{\mathcal{H}}}
\def\gI{{\mathcal{I}}}

\def\gM{{\mathcal{M}}}

\def\gO{{\mathcal{O}}}
\def\gP{{\mathcal{P}}}

\def\gS{{\mathcal{S}}}

\def\sP{{\mathbb{P}}}

\newcommand{\R}{\mathbb{R}}

\usepackage{amsmath}
\usepackage{amssymb}
\usepackage{mathtools}
\usepackage{amsthm}
\usepackage{booktabs} 
\usepackage{hyperref}
\usepackage[capitalize,noabbrev]{cleveref}
\usepackage{algorithm}
\usepackage{algorithmic}
\usepackage{pgfplots}
\usepackage{pgfplotstable}
\usepgfplotslibrary{groupplots}
\usepackage{xcolor}

\usepackage[table]{xcolor}
\usepackage{booktabs}
\usepackage{multirow}

\usepackage{float}
\theoremstyle{plain}
\newtheorem{theorem}{Theorem}[section]
\newtheorem{proposition}[theorem]{Proposition}
\newtheorem{lemma}[theorem]{Lemma}
\newtheorem{corollary}[theorem]{Corollary}
\theoremstyle{definition}
\newtheorem{definition}[theorem]{Definition}
\newtheorem{assumption}[theorem]{Assumption}
\theoremstyle{remark}

\usepackage[textsize=tiny]{todonotes}
\usepackage{xr-hyper}    
\usepackage{multirow}
\makeatletter
\let\@origaddcontentsline\addcontentsline
\newcommand{\StopWritingToTOC}{%
  \let\addcontentsline\@gobblethree
}
\newcommand{\StartWritingToTOC}{%
  \let\addcontentsline\@origaddcontentsline
}
\makeatother

\definecolor{lightblue}{RGB}{173, 216, 230}

\title{Leveraging Data Symmetries to Select an Optimal Subset of Training Data under Label Noise}

\author{{Kumar Shubham}}
\author{Pavan Karjol}
\author{ Kiran M K}
\author{Prathosh AP}

\affil{%
    Indian Institute of Science, Bangalore, India }

\StopWritingToTOC
\begin{document}
\maketitle

\begin{abstract}
 The performance of machine learning models often relies on large labeled datasets; however, data collected from diverse sources can contain label noise. Recent work has shown that, in noisy settings, there may exist a subset of the training data on which models can achieve performance comparable to training on a noise-free dataset. A widely used method for identifying such subsets is \textit{cutstats}, which employs $k$-nearest neighbors ($k$-NN) to detect low-noise samples. However, its performance on high-dimensional data remains largely unexplored. In this work, we formally establish that the performance of a classifier trained on a subset of a noisy dataset selected via \textit{cutstats} is influenced by the accuracy of $k$-NN. We further demonstrate that, in noisy environments, exploiting data invariance and knowledge of underlying symmetries can significantly enhance the performance of $k$-NN, bringing it closer to the Bayes optimal classifier even in high-dimensional regimes. Finally, we show that for real-world scenarios, where information about the underlying invariance is only partially known, learnt invariant representations can still facilitate the identification of near-optimal subsets. 

 \end{abstract}
\section{Introduction}

 For modern AI-based applications, large labeled datasets~\citep{imagenet,hypersim-data,celeb-500k} have become essential to achieve good performance. However, manually labeling such datasets is both costly and time-consuming~\citep{ratner2016data,chen2021building}. To alleviate this burden, recent works have proposed low-cost annotation strategies, including rules, heuristics, or lightweight models for generating approximate labels~\citep{distant_supervision_heuristic, distant_supervision_KB, PWS_survey}. While these approaches improve scalability, they frequently introduce label noise, which can substantially degrade the performance of the model.

Several techniques have been proposed to improve the model's performance under noisy labels~\citep{ghosh2017robust, wang2019imae, wang2019symmetric, yi2022learning, chen2024erase, zhou2024l2b, kim2024learning, gao2021searching, liao2022empirical}. However, these methods primarily focus on enhancing the robustness of the model, while assuming that the noisy training data remains fixed and cannot be altered. In practice, however, datasets often carry significant proprietary value~\citep{xiong2022recognition} and may be reused across multiple tasks or even traded as assets~\citep{de2015data,zech2017data}. This makes it crucial to address label noise directly at the data level, where identifying and removing corrupted samples yields cleaner subsets that can be reused across diverse tasks.

To address this, \citet{subset_ws_neurips22} recently demonstrated that, even in the presence of label noise, there exists an optimal subset of training data on which models can achieve performance comparable to their noiseless counterparts, subject to certain concentration bounds. To identify such a subset, they proposed \textit{cutstats}~\citep{muhlenbach2004identifying}, a $k$-nearest neighbor ($k$-NN)–based technique that selects candidates by comparing each sample’s features and labels with those of its neighborhood. Unlike other data-centric approaches~\citep{jiang2018mentornet,goel2022pars,li2020dividemix}, this method relies solely on the training data and its feature space, thereby avoiding any prior knowledge of the downstream classifier that will eventually be trained on the selected subset.

However, despite its promise, the effectiveness of \textit{cutstats} on high-dimensional data, and its reliability in recovering the theoretically optimal subset of training data~\citep{subset_ws_neurips22}, remains largely unexplored. In this work, we investigate the performance gap between classifiers trained on noisy subsets and those trained on noiseless datasets in high-dimensional settings. 

We show that, for \textit{cutstats}, this gap is determined by the accuracy of $k$-NN in identifying noiseless samples. In low-dimensional regimes, $k$-NN can approximate the Bayes classifier\footnote{This equivalence with the Bayes classifier holds for specific noise rates and with certain probability, as shown by \citet{bahri2020deep}. } (the theoretically optimal classifier) and thereby recover optimal subsets. In high-dimensional regimes, however, this approximation may not hold, and hence can lead to the selection of suboptimal subsets.

To address this limitation, we show that incorporating prior knowledge of symmetries and their associated invariances can preserve the effectiveness of $k$-NN in high-dimensional settings. Invariances like rotation or permutation can capture the structured patterns in the data and, therefore, can be leveraged to filter out noisy samples. By constructing representations that explicitly encode these invariances, we theoretically establish that the approximation with the Bayes classifier is preserved for orthogonal or permutation-based symmetries, thereby enabling the recovery of optimal subsets in higher dimensions. To extend this idea to real-world scenarios, where information about underlying symmetries is often incomplete or entirely unknown, we show that even partial knowledge of invariances or invariant representations learned through methods such as group-invariant representation learning~\citep{winter2022unsupervised} or contrastive learning~\citep{chen2020simple} can substantially enhance the performance of \textit{cutstats}.

Our contribution can be summarized as follows: 

\begin{itemize}
\item We show that the performance gap of a classifier trained on a subset of the training data selected using \textit{cutstats} is governed by the accuracy of $k$-nearest neighbors ($k$-NN) under label noise. Our analysis further shows that $k$-NN can approximate the Bayes classifier in such settings, thereby enabling optimal subset selection.
\item We provide a theoretical analysis of how this approximation with the Bayes classifier breaks down in high-dimensional regimes, where optimal subset selection is no longer guaranteed.
\item We demonstrate that exploiting symmetries and associated invariances, specifically orthogonal and permutation symmetries, can restore this approximation and ensure the recovery of optimal subsets.

\item For practical settings, we demonstrate that invariant representations learned through group-equivariant networks or contrastive learning can significantly enhance the performance of \textit{cutstats}, enabling the selection of accurate subsets of the dataset.   
\end{itemize}

\section{Related Work}

\subsection{Handling Label Noise}

 A common approach to mitigate the effect of noise in training datasets is to improve the robustness of the model. One widely used strategy for this task is to train models with noise-aware loss functions~\citep{ghosh2017robust,zhang2018generalized,wang2019symmetric,gao2021searching}, which prevent overfitting to noisy labels and enable better generalization to unseen samples.  However, these methods often suffer from optimization challenges and require extensive hyperparameter tuning, making them difficult to apply in practice. Moreover, they overlook the proprietary value of data and do not provide a mechanism for removing noisy samples from the dataset.

Unlike loss function-based techniques, recent data-centric approaches~\citep{jiang2018mentornet,goel2022pars} aim to improve model robustness not only by incorporating loss-based strategies but also by selecting appropriate training samples. Many of these methods~\citep{han2018co,yu2019does,zhang2021understanding,malach2017decoupling} adopt a co-training strategy that seeks to identify appropriate samples to improve the robustness of the model. On a similar line, self-supervised learning(SSL) techniques improve the performance of the model~\cite{zhang2019metacleaner} using relabeling via co-teaching~\cite{mandal2020novel}, and selecting or filtering samples during training~\citep{zheltonozhskii2022contrast,tan2021co}.  However, such approaches often require training multiple networks, which leads to increased memory usage and higher computational costs. 

To address these limitations, various strategies have been proposed to select a clean subset of the training dataset. Recent work has shown that, in the presence of noise, a representative subset of data points~\citep{guo2022deepcore} or samples with low predictive uncertainty~\citep{subset_ws_neurips22} can be selected from the training set. Similarly, identifying samples that are frequently forgotten during training can also be used to remove detrimental noisy samples~\citep{toneva2018empirical}, thereby enabling the selection of high-quality subsets of the training dataset. Among these approaches, \textit{cutstats}~\citep{subset_ws_neurips22} stands out for its simplicity, as it selects subsets of the training data solely based on input features and their associated labels, without any dependency on the downstream classifier. Such methods are particularly relevant in enterprise and crowdsourcing settings~\citep{PWS_survey}, where datasets are curated without prior knowledge of the downstream models that will be trained on them. In these scenarios, classifier-agnostic subset selection is essential to ensure broad applicability and robustness across diverse learning tasks. However, despite its practical advantages, the effect of data dimensionality on the performance of \textit{cutstats} remains largely unexplored.

\subsection {Classifier Performance and Group Invariance}
Real-world datasets often exhibit different forms of symmetry, where the underlying properties of the dataset remain unaltered for translation, rotation, and permutation~\citep{bronstein2021geometric,lecun2004learning} based transformations. Exploiting these symmetries can significantly improve the performance of machine learning models, particularly in domains where structured patterns are present. Recent advancements in equivariant neural networks have demonstrated that encoding symmetries, associated with rotation or permutation-based transformations, into model architecture can enhance its performance~\citep{ghimire2020learning,jiang2023anatomical,bronstein2021geometric}. In contrast to these approaches, we propose a data-centric deletion method that leverages the underlying invariance in a dataset to select an optimal training subset in the presence of label noise.

\section{Methods}
\subsection{Problem Definition}
\label{sec:prob_defn}
Let the noiseless dataset $\gD =\{(x_i,y_i)\}_{i=1}^s $, where each input $x_i \in \mathbb{R}^d$ is a $d$-dimensional feature vector drawn from an underlying distribution $\gP$, and the corresponding label ($y_i$) is generated by true-labeling function $\vf: \mathbb{R}^d \to \{0,1\}$ such that $y_i = \vf(x_i)$.

Suppose the data admits an underlying symmetry, represented by a group $G$ such that class labels are invariant under the group action, i.e.,: 
\[
\vf(g \cdot x) = \vf(x) \quad \forall g \in G,
\]
where ($\cdot$) denotes group action.  
For a group $G$, the orbit of any input $x$ is defined as 
\[
\gO_x = \{ g \cdot x : g \in G \}.
\] 
Let $\gF$ be an invariant representation function associated with $G$, such that all points in the same orbit share the same representation: 
\[
\gF(x) = \gF(x') \quad \text{whenever } x' \in \gO_x.
\] 
Moreover, $\gF$ retains all class-relevant information ~\citep{achille2018emergence,achille2018information}.

In practice, the true labeling function $\vf$ is not observable, and instead we have access to noisy approximations of the labels.  
Let $\gD_{\text{noisy}} = \{(x_i, \hat{y}_i) \mid i=1,\dots,s \}$, where $\hat{y}_i$ denotes a noisy version of the ground-truth label.  

The goal of this work is to identify a subset $\gS \subseteq \gD_{\text{noisy}}$ using the \textit{cutstats} procedure, while leveraging the invariant representation function $\gF$, such that a classification model $\gM$ trained on $\gS$ approximates the true labeling function($\vf$).


\begin{figure}[ht!]
\centering
    \includegraphics[width=0.47\textwidth]{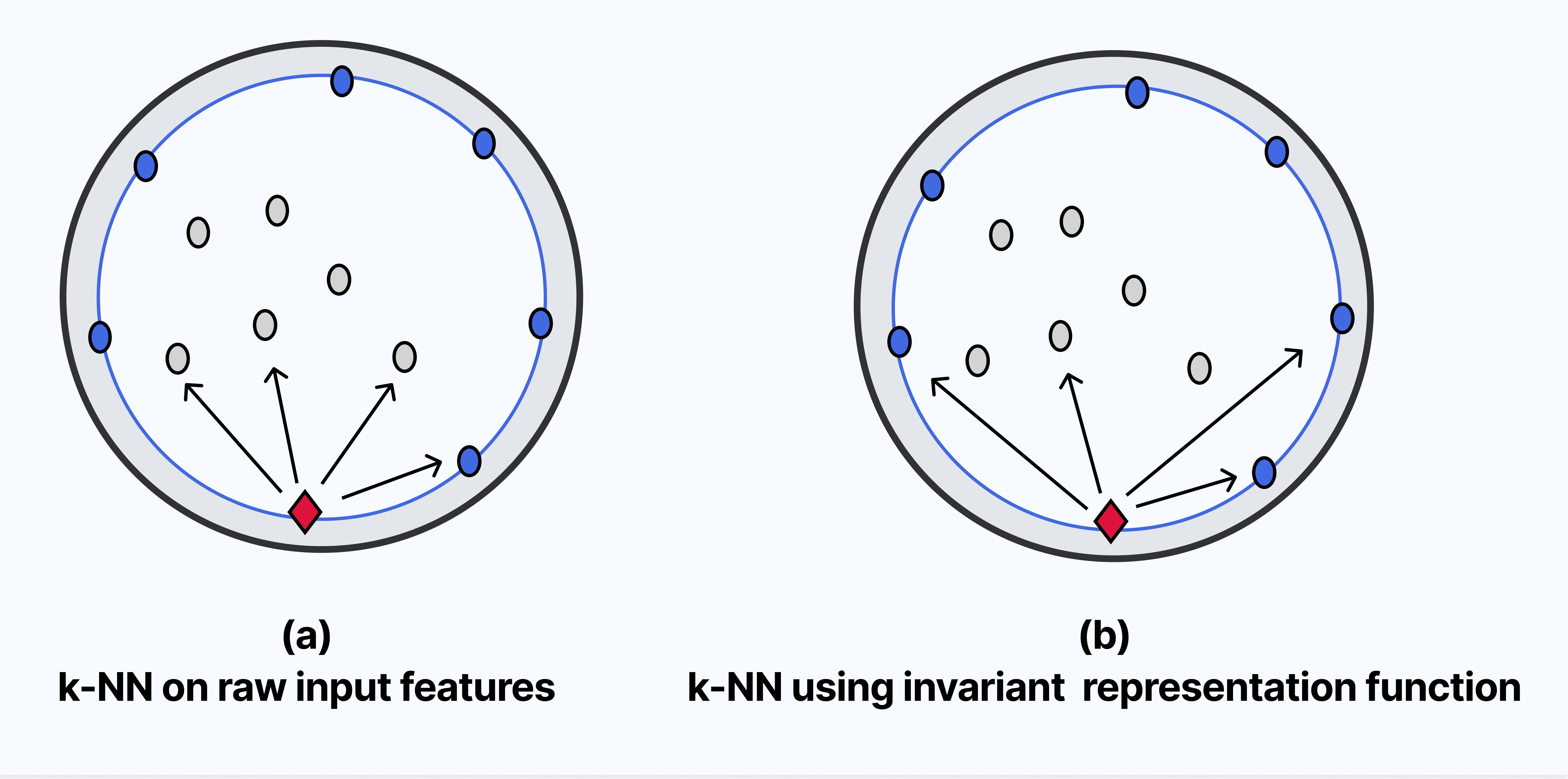}
    \caption{Given image depicts the Nearest neighbors selected by different methods under Orthogonal group invariance in the dataset. Data points near the circumference of the circle constitute an orbit and can share the same class labels. (a) The method that utilizes standard k-NN using input features can select the nearest neighbors that do not belong to the orbit(Section~\ref{sec:prob_defn}). (b) The method that utilizes an invariant representation function can select the nearest neighbor from the orbit. }
    \label{fig:orbit}
\end{figure}

\subsection{Methodology}
To achieve this objective, we first theoretically analyze the error gap of a model trained using a subset of a noisy dataset with its noiseless counterpart. 
For practical scenarios where the noiseless datasets are not available, we show that the resulting error gap for \textit{cutstats}, depends on the accuracy of $k$-NN in identifying correct samples from the noisy dataset. 
In low-dimensional settings, it can closely approximate the Bayes-optimal classifier and, consequently, achieve near-optimal classification performance. However, this approximation might not hold as the dimensionality increases. 

We then demonstrate that incorporating the invariant representation $\gF$ into subset selection can restore this approximation in high-dimensional regimes, particularly for orthogonal and permutation groups. 
n practical scenarios where the underlying invariances are unknown or only partially specified, we show that group-invariant representation learning or contrastive learning can provide effective approximations. 
In the following section, we formally discuss the \textit{cutstats} technique in detail and describe how knowledge of $\gF$ can be integrated into the subset selection process.
 
\subsection{Cutstats}
\label{sec:cutstats}
 \textit{Cutstats} provides an efficient technique for selecting a subset of the noisy training dataset ($\gD_{noisy}$) based on the agreement between the noisy class label ($\hat{y}_i$) of a sample ($x_i$) and the labels of its nearest neighbors $\text{NN}(x_i)$. In \textit{cutstats}, the selection of the subset $\gS \subseteq \gD_{noisy}$ is based on the $Z$-score ($z_i$), defined as:

\begin{align}
z_i &= \frac{J_i - \mu_i}{\sigma_i}, \nonumber \\
J_i &= \sum\limits_{j \in \text{NN}(x_i)} w_{i,j} \vone\big(\hat{y}_j \neq \hat{y}_i\big), \nonumber \\
w_{i,j} &= \big(1 + \|x_i - x_j\|_2\big)^{-1}, \nonumber \\
\mu_i &= \big[1 - \sP(\hat{y}_i)\big] \sum_{j \in \text{NN}(x_i)} w_{i,j}, \nonumber \\
\sigma_i^2 &= \sP(\hat{y}_i)\big[1 - \sP(\hat{y}_i)\big] \sum_{j \in \text{NN}(x_i)} w_{i,j}^2
\label{eqn:z-score}
\end{align}

where the weights $w_{i,j}$ capture feature similarity between $x_i$ and its neighbors $x_j \in \text{NN}(x_i)$. The score $J_i$ measures the weighted disagreement with neighbors’ labels, while $\mu_i$ and $\sigma_i^2$ denote the expected mean and variance of $J_i$ under the null hypothesis of independent label assignment, estimated using the label distribution \big($\sP(\hat{y})$\big) in the training dataset $\gD_{noisy}$.

Small $z_i$ values indicate strong agreement between a sample and its neighborhood, whereas large $z_i$ values suggest disagreement between a sample’s label and that of its nearest neighbor, signaling potential label noise.

While the classical  \textit{cutstats} uses raw input features~\citep{subset_ws_neurips22} to define weights and nearest neighbors. The invariant representation function ($\gF$) can be incorporated for calculating the weights and identifying the nearest neighbor as follows:

\begin{align}
w_{i,j} &= \big(1 + \| \gF(x_i) - \gF(x_j) \|_2\big)^{-1}
.
\end{align}

where the neighborhood is now defined as   $NN(x_i) = \{\, x_j : x_j$  is a nearest neighbor of  $x_i$  with respect to the  $\ell_2$  distance in the representation space $ \gF \}$. 

For subset selection, samples are ranked by $z_i$, and the top $\tau\%$ with the smallest values are retained in $\gS$, while the remaining samples are excluded.

Once a subset of the noisy training dataset has been selected, a downstream classifier ($\gM$) is trained on it and subsequently used to predict the true class labels of new, unseen samples.

For ease of our theoretical discussion, we assume that  \textit{cutstats} considers $k$-nearest neighbors (k-NN) with balanced class priors, i.e., $\sP(\hat{y}=0) = \sP(\hat{y}=1) = \tfrac{1}{2}$. Since  \textit{cutstats} assigns smaller $z$-values when the majority of neighbors share the same class as the selected data point, our analysis will primarily focus on the nearest neighbors and their labels. Under the given assumption, we will next analyze the generalization performance of the selected subset of the noisy dataset.


\subsection{Performance of Selected Subset and Noiseless Dataset}
The generalization gap between a classifier trained on a subset of noisy data and those trained on the noiseless data has been formally characterized by \citet{subset_ws_neurips22}. In their framework, a subset selection heuristic (e.g.,  \textit{cutstats}) may abstain from including a training sample $(x,\hat{y})$ in subset ($\gS$) by assigning it the special label $\phi$ (i.e., $\hat{y}=\phi$, where $\hat{y} \in \{0,1,\phi\}$). In such a setting, \citet{subset_ws_neurips22} formally  present a bound on the performance of a classifier using balanced error ($err_{bal}$), defined as : 
\[
\scalebox{0.8}{$\displaystyle err_{bal}(\mathcal{M},y)=\frac{1}{2}\bigg(\mathbb{P}(\mathcal{M}(x)=1 \mid y=0)+\mathbb{P}(\mathcal{M}(x)=0 \mid y=1)\bigg)$}
\]

For a classifier $\gM$ trained on noisy labels $\hat{y}$, this balanced error can be compared against that of an optimal classifier $\gM^*$ trained on noiseless data. The performance gap largely depend upon the non-abstention rate $\sP(\hat{y} \neq \phi)$ and the noise rates of the selected subset ($\alpha, \gamma$). The complete formulation is provided below :

\begin{lemma}
{\citep{subset_ws_neurips22}}
\label{Thm:subset_conc_bound}
Under the assumptions outlined by \citet{subset_ws_neurips22}. Let $\Gamma$ denote the hypothesis class of the downstream model used for training the classifier. Furthermore, let $\gM^*$ represent the optimal classifier trained on noiseless class labels \((y)\), defined as 
\[
    \gM^* = \inf_{\gM \in \Gamma} \text{err}_{\text{bal}}(\gM, y),
\]
 and \(\gM\) denote the classifier obtained by minimizing the empirical balanced error on the noisy dataset ($\gD_{noisy}$) with m samples. Then, with probability \(1 - \delta\),
\begin{align}
    &\text{err}_{\text{bal}}(\gM, y) - \text{err}_{\text{bal}}(\gM^*, y) \notag \\
    &\leq \tilde{O} \left(\frac{1}{1 - \alpha - \gamma} 
    \sqrt{\frac{\text{VC}(\Gamma) + \log \frac{1}{\delta}}{m  \sP[\hat{y} \neq \phi]  \min_y \sP[\hat{y}  = y \mid \hat{y} \neq \phi]}} \right),
\end{align}
where \(\tilde{O}\) hides logarithmic factors in \(m\), \(\text{VC}(\Gamma)\), and 
\[
\alpha =  \sP(y=0 | \hat{y}=1), \gamma = \sP(y=1 | \hat{y}=0)
\]denoted the error rates in the selected subset. 
\end{lemma}

As stated in Lemma~\ref{Thm:subset_conc_bound}, the performance of any downstream model trained on a noisy subset of data depends on two primary factors:  
(i) the level of noise present in the selected set ($\alpha, \gamma$), and  
(ii) the effective size of the dataset used for training, \big($ m \sP[\hat{y} \neq \phi] $\big).  

For any given noisy dataset with m training samples ($\sP(\hat{y} \neq \phi) = 1$) and the error rates $\alpha_{\text{noisy}}, \gamma_{\text{noisy}}$, an appropriate subset $\gS$ can be selected by abstaining from the noisy samples. This results into generation of a new dataset $\gS$ with error rates defined as $\alpha_{\gS}, \gamma_{\gS}$. For $k$-NN based methods, like \textit{cutstats}, the formal relationship between the error rates of the original data and the subset can be characterized as follows:


\begin{proposition}[Performance of $k$-NN]
\label{alg:performance}
Consider the setup of Lemma~\ref{Thm:subset_conc_bound}. 
 for $k$-NN ($\gH : x \mapsto y^{\text{knn}}$), with $y^{knn}$ as the class label of the majority of the samples in the nearest neighbor and
\[
\lambda_0 = \sP(y^{\text{knn}} = 0 \mid y=0), 
\qquad 
\lambda_1 = \sP(y^{\text{knn}} = 1 \mid y=1)
\] 
denote the class-wise accuracies of $k$-NN for class $0$ and class $1$, respectively.  
Then the error rate of the subset selected by $k$-NN satisfies:
\begin{align}
    \alpha_{\gS} &= \frac{\alpha_{noisy} (1 - \lambda_0)}{\alpha_{noisy} (1 - \lambda_0) + (1 - \alpha_{noisy})\lambda_1}, \nonumber \\ 
    \gamma_{\gS} &= \frac{\gamma_{noisy} (1 - \lambda_1)}{\gamma_{noisy} (1 - \lambda_1) + (1 - \gamma_{noisy})\lambda_0} \nonumber
\end{align}
\end{proposition}


Proofs of the propositions and corollaries are provided in the Appendix~\ref{proof:acc_heuristic}.

Intuitively, more accurate heuristics (with $\lambda_0, \lambda_1$ close to 1) introduce less noise during subset selection. This relationship is formally described below:

\begin{corollary}
\label{corr_main:acc_performance}
Suppose the class priors are balanced, i.e., $\sP(y=0) = \sP(y=1) = \tfrac{1}{2}$. If the $k$-NN classifier has error $\sP(y^{\text{knn}} \neq y) \leq \tfrac{1}{2}$, then the subset error rates satisfy
\[
    \alpha_{\gS} \leq \alpha_{noisy}
    \quad \text{and} \quad
    \gamma_{\gS} \leq \gamma_{noisy}.
\]
  Moreover, as $\lambda_{0}, \lambda_{1} \to 1$, the subset error rates also reduces:
\[
    \alpha_{\gS}, \, \gamma_{\gS} \;\to\; 0.
\]
\end{corollary}

Corollary~\ref{corr_main:acc_performance} and Lemma~\ref{Thm:subset_conc_bound} establish a relationship between the precision of $k$-NN in the presence of label noise and that of a downstream classifier trained on a subset of data selected by it. Since, in theory, the best performance achievable by any classifier corresponds to that of the Bayes classifier, in the next section we analyze the conditions under which $k$-NN can approximate the Bayes classifier in a noisy setting.

\subsection{Bayes optimal classifier and K-NN}

The relationship between the performance of $k$-NN and the Bayes classifier in the presence of label noise has been formally established by \citet{bahri2020deep}. According to their analysis, if a dataset contains $n$ uncorrupted samples and $\gC$ samples with label noise. Then, for  $k$-nearest neighbors classifier ($\eta_k(x)$), and the Bayes-optimal classifier ($\eta^*(x)$) following holds: 

\begin{lemma}\citep{bahri2020deep}
\label{main_thm:rate_tsyb_original}
Let $\nu > 0$, and assume the conditions of \citet{bahri2020deep} hold under the Tsybakov noise condition with parameter $\beta$. Then there exist constants $K_l(d), K_u(d), K, K' > 0$, depending only on the dimension $d$ and the underlying distribution, such that with probability at least $1 - \nu$ :

 If \( k \) lies in the range


\[
K_l(d) \cdot \log^2(1/\nu) \cdot n^{\frac{\rho}{\rho+d}} \leq k \leq K_u(d) \cdot \Delta(\mathcal{C})^d \cdot n,
\] Then,
\[
\mathbb{P} \big(\eta_k(x) \neq \eta^*(x)\big) \leq K \cdot \lambda^\beta,
\]
\[
R_X - R^* \leq K' \cdot \lambda^{\beta+1},
\]
where
\[
\lambda = \left( \sqrt{\frac{\log n + \log(1/\nu)}{k}} + \left( \frac{k}{n} \right)^{\rho/d} \right),
\]
  $R_X$ and $R^*$ denote the risks of the $k$-NN and the Bayes optimal classifier, respectively.
\end{lemma}
where \(\Delta(\mathcal{C})\) denotes the minimum distance between features of corrupted examples and $\rho$ is a parameter associated with Holder's continuity. Details about all the parameters and the associated assumptions are provided in the Appendix.

An important consequence of Lemma~\ref{main_thm:rate_tsyb_original} is that, under noisy conditions, there exists a range of values of $k$ for which the $k$-NN converges to the Bayes optimal classifier. In such cases, $k$-NN can approximate an optimal classifier and can therefore identify an optimal subset of data points without incurring the additional computational cost of training a large model for this task. However, in high-dimensional settings, this approximation with the Bayes classifier may not hold.

In fact,  when the feature dimension $d$ is large, there does not exist any value of $k$ for which the approximation with Bayes classifier remains valid. 

\begin{proposition}[High-dimensional $k$-NN and Bayes Classifier ]
\label{main_prop:impossible_condn}
As the feature dimension $d$ increases, the factor $K_u(d) \cdot \Delta(\mathcal{C})^d \cdot n$ in Lemma~\ref{main_thm:rate_tsyb_original} decreases with $d$ \big($K_u(d)$  decreases super exponentially with $d$ \big). Consequently, there exists a threshold $d_0$ such that for all data dimension $d \geq d_0$ and for a  given n, $\nu$, $\rho$ and $\Delta$ no $k$ satisfies
\[
    K_l(d) \cdot \log^2(1/\nu) \cdot n^{\frac{\rho}{\rho+d}} \leq k \leq K_u(d) \cdot \Delta(\mathcal{C})^d \cdot n.
\]
\end{proposition}

Proof of the given proposition is provided in Appendix~\ref{appx_proof_bayes}.

As per the given Proposition~\ref{main_prop:impossible_condn}, the $k$-NN model may not approximate the performance of the Bayes optimal classifier for high-dimensional data. This limitation reduces the efficacy of $k$-NN in high-dimensional settings, as it may not guarantee performance comparable to the Bayes optimal classifier when no such $k$ exists.
  
\subsection{Orthogonal and Permutation Group}
To address this problem in high-dimensional settings, we show that for datasets 
whose labeling function ($\vf$) is invariant under the $\mathrm{SO}(d)$ or $S_d$ 
group (Section~\ref{sec:prob_defn}), incorporating an invariant representation 
function ($\gF$) into the nearest-neighbor classifier (Section~\ref{sec:cutstats}) 
yields performance that remains approximately as good as the Bayes-optimal 
classifier even in high dimensions.

\begin{proposition}[Orthogonal Group]
\label{main_prop:orthogonal_grp}
For data whose true-labeling function  satisfy orthogonal invariance ($SO(d)$) as per section~\ref{sec:prob_defn}, if $k$-NN utilizes the invariant representation function ($\gF$), and the inequality holds for dimension ($d$=1), i.e., 
\[
    K_l(1) \cdot \log^2(1/\nu) \cdot n^{\frac{\rho}{\rho+1}} \leq k \leq K_u(1) \cdot \Delta(\mathcal{C}) \cdot n.
\]

Then there exists a $k$  for all dimensions $d$ for which $k$-NN will approximate Bayes optimal classifier for a  given $n$, $\nu$, $\rho$ and $\Delta$,
\end{proposition}

\begin{proposition}
\label{main_prop:perm_grp}
For data whose true-labeling function  satisfy permutation  invariance ($S_d$) as per section~\ref{sec:prob_defn}, if $k$-NN utilizes the invariant representation function ($\gF$), then, $K_u(d) \cdot \Delta(\mathcal{C})^d \cdot n$ does not decrease with d, and there will exist $k$ for a  given $n$, $\nu$, $\rho$ and $\Delta$ for which
\[
    K_l(d) \cdot \log^2(1/\nu) \cdot n^{\frac{\rho}{\rho+d}} \leq k \leq K_u(d) \cdot \Delta(\mathcal{C})^d \cdot n.
\] holds, and $k$-NN can approximate Bayes optimal classifier. 
\end{proposition}

Proof of the given propositions is provided in Appendix~\ref{appx_proof_bayes}. For high-dimensional data, even though the desired approximation of the Bayes optimal classifier (Lemma~\ref{main_thm:rate_tsyb_original}) cannot be achieved directly from the input features $x$, as established by Proposition~\ref{main_prop:impossible_condn}, Propositions~\ref{main_prop:orthogonal_grp} and \ref{main_prop:perm_grp} demonstrate that this approximation can be recovered by employing an invariant representation function $\gF$, which enables the selection of informative subsets of the training data for downstream classifier training. 

\subsection {Invariant Representation Function}
A key requirement of the proposed approach is access to an invariant representation function $\gF$, as discussed in Section~\ref{sec:prob_defn}, which preserves the class-relevant information and therefore yields the same bayes error as the original input features, i.e., $\sP_e(y,\gF(x)) = \sP_e(y,x)$. However, in many real-world scenarios, the exact invariant representation is not known and should be approximated from data. Such representations are commonly obtained using methods including unsupervised group-invariant representation learning~\citep{winter2022unsupervised} and contrastive learning~\citep{chen2020simple}, which aim to capture the intrinsic invariances present in the data without relying on ground-truth labels. While these representations can facilitate the efficient utilization of the underlying dataset, they may not capture the underlying class-relevant information entirely, resulting in higher Bayes error~\citep{garcia2012divergences,goel2003one}. Nevertheless, in practice, as discussed in the subsequent section, such representations often yield strong downstream performance. This is likely because they suppress noise and capture underlying invariances in the data. While such representations facilitate more efficient utilization of the dataset, their performance is governed by how faithfully they retain discriminative information necessary for accurate classification~\citep{garcia2012divergences,goel2003one}.

\section{Experiments}
 \begin{table*}[th!]
    \centering
    \caption{Details about the group, associated invariant function, and generating function used for the experiment. we have used norm operator as invariant function for Orthogonal group and sort function for Permutation group where $\operatorname{sort}(\{f_1, \ldots, f_n\}) = (f_{(1)}, \ldots, f_{(n)}) \quad \text{where } f_{(1)} \le \cdots \le f_{(n)}$.  }
    \renewcommand{\arraystretch}{1.5} 
    \resizebox{0.7\textwidth}{!}{
    \begin{tabular}{ccccc}
        \toprule
        \textbf{Group name} & \textbf{Invariant function} & \textbf{Generating function}  \\
        \midrule
        Orthogonal Group & $\gF(x)\rightarrow ||x||_2$ & $h(x) = k_1\sin(c_1x^\top x) + k_2\sin^{2}(c_2x^\top x) + k_3\cos(c_3x^\top x)$\\ 
        Permutation Group & $\gF(x)\rightarrow \operatorname{sort}(x)$ & $h(x) = \sum_{k=1}^{l} \sum_{i=1}^{d} \sin(x_i^k) $\\ 
        \bottomrule
    \end{tabular}
    }

    \label{table:group_data}
\end{table*}

\begin{table*}[ht!]
\centering
\caption{Comparison of downstream classifier accuracy and subset accuracy across different methods. Results are reported as mean across three different seeds, with the best subset selection method highlighted in \textbf{bold}.}
\label{tab:perm_ortho_comparison}
\resizebox{0.6\textwidth}{!}{%
\begin{tabular}{@{}lcccc@{}}
\toprule
\multirow{2}{*}{\textbf{Method}} & \multicolumn{2}{c}{\textbf{Orthogonal Group}} & \multicolumn{2}{c}{\textbf{Permutation Group}} \\
\cmidrule(lr){2-3} \cmidrule(lr){4-5}
& \textbf{classifier-Acc} & \textbf{Subset-Acc} & \textbf{classifier-Acc} & \textbf{Subset-Acc} \\
\midrule
\rowcolor{cyan!3}
NoiseLess 
& $70.40 \pm 0.83$ & $100.00 \pm 0.00$ 
& $90.93 \pm 0.30$ & $100.00 \pm 0.00$ \\
\hline
Noisy-Complete 
& $50.83 \pm 8.11$ & $55.08 \pm 0.11$ 
& $78.73 \pm 3.45$ & $55.33 \pm 0.82$ \\
Random 
& $45.72 \pm 3.26$ & $54.93 \pm 0.57$ 
& $77.33 \pm 4.77$ & $55.50 \pm 0.49$ \\
Vanilla Cutstats 
& $63.63 \pm 4.81$ & $60.51 \pm 1.58$ 
& $82.70 \pm 4.37$ & $71.21 \pm 3.09$ \\
Entropy 
& $62.53 \pm 6.64$ & $56.46 \pm 0.33$ 
& $70.72 \pm 4.61$ & $56.05 \pm 1.36$ \\
Forget 
& $68.45 \pm 0.53$ & $58.69 \pm 1.13$ 
& $75.98 \pm 2.88$ & $63.04 \pm 1.96$ \\
Herding 
& $51.32 \pm 0.98$ & $55.66 \pm 0.51$ 
& $71.00 \pm 2.35$ & $54.13 \pm 0.77$ \\
\midrule
\rowcolor{cyan!3}
\textbf{Ours} 
& \textbf{68.52 $\pm$ 0.66} 
& \textbf{74.31 $\pm$ 1.10} 
& \textbf{86.00 $\pm$ 0.84} 
& \textbf{74.42 $\pm$ 3.15} \\
\bottomrule
\end{tabular}
}
\end{table*}

\begin{table*}[htbp]
\centering
\caption{Performance comparison on Rotated MNIST dataset under different noise probabilities.}
\label{tab:rotmnist_results}
\resizebox{0.85\textwidth}{!}{%
\begin{tabular}{@{}lcccccc@{}}
\toprule
\multirow{2}{*}{\textbf{Method}} & 
\multicolumn{2}{c}{\textbf{Noise 0.4}} & 
\multicolumn{2}{c}{\textbf{Noise 0.6}} & 
\multicolumn{2}{c}{\textbf{Noise 0.8}} \\
\cmidrule(lr){2-3} \cmidrule(lr){4-5} \cmidrule(lr){6-7}
& \textbf{Classifier-Acc} & \textbf{Subset-Acc} & \textbf{Classifier-Acc} & \textbf{Subset-Acc} & \textbf{Classifier-Acc} & \textbf{Subset-Acc} \\
\midrule
\rowcolor{cyan!3}
\rowcolor{cyan!3}
NoiseLess 
& $96.80 \pm 0.27$ & $100.00 \pm 0.00$ 
& $96.80 \pm 0.27$ & $100.00 \pm 0.00$ 
& $96.80 \pm 0.27$ & $100.00 \pm 0.00$ \\
\hline

Noisy-Complete 
& $86.63 \pm 1.13$ & $59.98 \pm 0.28$ 
& $79.51 \pm 8.13$ & $39.94 \pm 0.09$ 
& $60.27 \pm 3.49$ & $20.07 \pm 0.16$ \\

Random 
& $87.68 \pm 2.08$ & $59.99 \pm 0.34$ 
& $60.98 \pm 0.42$ & $39.88 \pm 0.25$ 
& $37.60 \pm 8.52$ & $20.07 \pm 0.33$ \\

Vanilla Cutstats 
& $90.86 \pm 1.77$ & $99.65 \pm 0.06$ 
& $83.78 \pm 6.52$ & $86.51 \pm 0.08$ 
& $68.93 \pm 2.87$ & $35.33 \pm 0.45$ \\

Entropy 
& $80.88 \pm 5.96$ & $72.64 \pm 0.79$ 
& $69.79 \pm 1.54$ & $47.18 \pm 0.69$ 
& $36.59 \pm 2.90$ & $21.85 \pm 0.28$ \\

Forget 
& $65.89 \pm 0.14$ & $54.95 \pm 0.51$ 
& $42.63 \pm 3.27$ & $28.19 \pm 1.26$ 
& $2.93 \pm 1.44$ & $8.51 \pm 0.67$ \\

Herding 
& $87.01 \pm 1.82$ & $62.51 \pm 0.52$ 
& $79.56 \pm 4.17$ & $41.08 \pm 0.39$ 
& $39.93 \pm 3.11$ & $20.24 \pm 0.10$ \\

\midrule
\rowcolor{cyan!3}
\textbf{Ours (contrastive)} 
& $83.95 \pm 2.04$ & $78.16 \pm 0.18$ 
& $76.28 \pm 0.57$ & $53.30 \pm 0.12$ 
& $41.52 \pm 3.71$ & $24.24 \pm 0.20$ \\

\rowcolor{cyan!3}
\textbf{Ours (group inv)} 
& \textbf{91.81 $\pm$ 0.58} & \textbf{99.82 $\pm$ 0.02} 
& \textbf{89.37 $\pm$ 1.83} & \textbf{90.51 $\pm$ 0.44} 
& \textbf{75.73 $\pm$ 5.99} & \textbf{37.77 $\pm$ 0.41} \\

\bottomrule
\end{tabular}
}
\end{table*}
\begin{table*}[htbp]
\centering
\caption{Performance comparison on Tetris dataset under different noise probabilities.}
\label{tab:tetris_results}
\resizebox{0.85\textwidth}{!}{%
\begin{tabular}{@{}lcccccc@{}}
\toprule
\multirow{2}{*}{\textbf{Method}} & 
\multicolumn{2}{c}{\textbf{Noise 0.4}} & 
\multicolumn{2}{c}{\textbf{Noise 0.6}} & 
\multicolumn{2}{c}{\textbf{Noise 0.8}} \\
\cmidrule(lr){2-3} \cmidrule(lr){4-5} \cmidrule(lr){6-7}
& \textbf{Classifier-Acc} & \textbf{Subset-Acc} & \textbf{Classifier-Acc} & \textbf{Subset-Acc} & \textbf{Classifier-Acc} & \textbf{Subset-Acc} \\
\midrule
\rowcolor{cyan!3}
NoiseLess 
& $88.30 \pm 0.33$ & $100.00 \pm 0.00$ 
& $88.30 \pm 0.33$ & $100.00 \pm 0.00$ 
& $88.30 \pm 0.33$ & $100.00 \pm 0.00$ \\
\hline

Noisy-Complete 
& \textbf{88.18 $\pm$ 0.33} & $59.95 \pm 0.64$ 
& \textbf{75.23 $\pm$ 1.55} & $40.09 \pm 0.59$ 
& $34.48 \pm 5.49$ & $20.21 \pm 0.24$ \\

Random 
& $54.47 \pm 4.14$ & $60.21 \pm 1.22$ 
& $44.93 \pm 1.00$ & $40.56 \pm 0.40$ 
& $24.18 \pm 5.55$ & $20.33 \pm 0.59$ \\

Vanilla Cutstats 
& $57.05 \pm 3.77$ & $60.69 \pm 1.04$ 
& $50.88 \pm 7.66$ & $40.48 \pm 0.72$ 
& $29.62 \pm 6.43$ & $20.72 \pm 0.31$ \\

Entropy 
& $33.80 \pm 5.96$ & $59.56 \pm 0.85$ 
& $37.35 \pm 0.33$ & $40.32 \pm 1.18$ 
& $28.63 \pm 5.19$ & $19.95 \pm 0.49$ \\

Forget 
& $24.53 \pm 5.72$ & $38.05 \pm 6.07$ 
& $20.92 \pm 6.25$ & $26.21 \pm 7.58$ 
& $12.40 \pm 0.71$ & $15.00 \pm 2.03$ \\

Herding 
& $33.08 \pm 5.38$ & $59.73 \pm 1.61$ 
& $32.90 \pm 5.63$ & $40.58 \pm 0.52$ 
& $32.95 \pm 7.07$ & $20.22 \pm 0.56$ \\

\midrule
\rowcolor{cyan!3}
\textbf{Ours (contrastive)} 
& $52.12 \pm 1.86$ & $60.23 \pm 0.53$ 
& $50.90 \pm 2.52$ & $39.52 \pm 0.85$ 
& $30.18 \pm 4.46$ & $20.05 \pm 0.63$ \\

\rowcolor{cyan!3}
\textbf{Ours (group inv)} 
& $75.72 \pm 5.65$ & \textbf{100.00 $\pm$ 0.00} 
& 67.28 $\pm$ 6.33 & \textbf{95.11 $\pm$ 1.45} 
& \textbf{40.95 $\pm$ 5.34} & \textbf{36.17 $\pm$ 0.84} \\

\bottomrule
\end{tabular}
}
\end{table*}

\subsection{settings}
In this work, we conduct experiments to analyze five key aspects. First, we investigate whether knowledge of the underlying invariant function ($\gF$) improves the performance of \textit{cutstats} when the dataset satisfies orthogonal and permutation invariance. For this experiment, we use synthetic data and the invariant functions outlined in Table~\ref{table:group_data}. Second, we conduct experiments on real-world datasets like rotated MNIST~\citep{larochelle2007empirical} where the labeling function satisfies $\mathrm{SO}(2)$ invariance, but the underlying invariant function is not known and must be learned from the data. To further analyze the influence of the Bayes error, we consider two methods: (i) unsupervised group invariant representation learning~\citep{winter2022unsupervised}, which leverages the knowledge about underlying symmetries and associated group to learn an invariant representation function  \textbf{ours (group inv)}, and (ii) contrastive learning~\citep{chen2020simple}, which learns representations solely based on data augmentations \textbf{ours (contrastive)}. We further evaluate our approach on the Tetris dataset~\citep{winter2022unsupervised}, which satisfies $\mathrm{SE}(3)$ invariance, to analyze whether our findings extend to other groups. Finally, we analyze how errors in the learned representation influence the ability of \textit{cutstats} to identify accurate subsets of the training dataset. Additional experiments on the impact of dimension on the performance of K-NN, comparison with robustness-based baselines, and performance on datasets like CIFAR10-N are provided in Appendix~\ref{appx:high_dim}~\ref{appx:robust}~\ref{appx:cifar10}, respectively. 

For experiments with synthetic datasets, we select functions whose classes are invariant to orthogonal and permutation group actions. We generate inputs $x$ sampled from a uniform distribution and assign class labels by thresholding $h(x)$, as defined in Table~\ref{table:group_data}. Since the function $h(x)$ is invariant under group transformations, the resulting true class labels $y$ also remain invariant. To introduce label noise, we randomly flip class labels in the training data with probability $p$. As a downstream classifier, we use a two-layer neural network. Further details on the data generation process and training specifications are provided in the Appendix. 
For rotated MNIST and Tetris, we follow the same settings as outlined in \citet{winter2022unsupervised} for learning representations. For rotated MNIST, we use a convolutional network as the downstream classifier, while for Tetris, we employ a graph neural network. We compare our approach against the following baselines:  \textbf{NoiseLess}: the original dataset without label noise.  \textbf{Noisy-Complete}: the complete noisy dataset with the same features but corrupted labels,   \textbf{Random}: a random subset of the noisy dataset,  \textbf{Entropy}: a setting where high-entropy examples are excluded from training~\citep{subset_ws_neurips22},  \textbf{Vanilla cutstats}: \textit{cutstats} using raw input features for subset selection~\citep{subset_ws_neurips22},  \textbf{Forget}: a setting where frequently forgotten samples are removed during training~\citep{toneva2018empirical},  \textbf{Herding}: a setting where only representative samples of the dataset are retained~\citep{guo2022deepcore}. All subset selection methods utilize the same proportion of the data, (40\% of the noisy training dataset), with the only difference being the specific set of points included in the subset. We report both the accuracy of the downstream classifier (\textbf{Classifier-Acc}) and the percentage of correctly labeled samples in the selected subset (\textbf{Subset-Acc}). All experiments were conducted on two NVIDIA A6000 GPUs with 48GB of VRAM each.

\subsection{Orthogonal and Permutation Group}

Table~\ref{tab:perm_ortho_comparison} presents a comparison of downstream classifier accuracy and subset accuracy across different subset selection methods for the orthogonal and permutation groups. Since the dataset corresponds to a binary classification problem with only two classes, we use a high noise probability ($p=0.45$) for this experiment. As per the results,``NoiseLess'' baseline achieves perfect subset accuracy (100\%) and the highest downstream performance across both groups, serving as an upper bound with clean data for both classifier and subset accuracy. In contrast, the ``Noisy-Complete'' baseline highlights the detrimental impact of label noise on performance. Among all subset selection methods, our proposed approach (``Ours'') achieves the best performance, with classifier accuracies of 68.52\% and 86.00\% for the orthogonal and permutation groups, respectively. It substantially outperforms competing methods by up to 3.3\% in classifier accuracy for the permutation group and 13.8\% in subset accuracy for the orthogonal group, demonstrating its effectiveness in identifying clean, reliable subsets from noisy data while maintaining competitive classification performance relative to the noise-free upper bound.

\subsection{Experiment on Rotated MNIST}
Table~\ref{tab:rotmnist_results} presents a comparison of the baselines on the rotated MNIST dataset across three different noise levels (0.4, 0.6, and 0.8). The invariant group representation (``Ours (group-inv)'') consistently achieves the best downstream classifier accuracy (91.81\%, 89.37\%, and 75.73\%, respectively) and subset accuracy (99.82\%, 90.51\%, and 37.77\%) across all noise levels, demonstrating superior robustness. In contrast, representations learned via contrastive learning underperform in these settings, possibly because of a higher Bayes error compared to the group-invariant representation, which leverages knowledge of the underlying group structure. These results illustrate that learning an appropriate invariant representation that preserves class-relevant information is critical for achieving strong performance.

\subsection{Experiment on Tetris}
 Table~\ref{tab:tetris_results} presents the results for the Tetris dataset. The proposed method (``Ours (group-inv)'') demonstrates robustness to noise even on datasets with other symmetries, achieving classification accuracies of 75.72\%, 67.28\%, and 40.95\% across the three noise levels, respectively. These results highlight its superior capability in handling label noise in settings where other subset selection methods largely fail.

\subsection{Noisy Representation}

To analyze scenarios where the learned representation does not fully capture the underlying symmetry, we conducted an ablation for the orthogonal group. Specifically, we examined how the \textit{invariance error}~\citep{moskalev2023genuine}, defined as $\mathbb{E}\left(\left| h(x) - h(g \cdot x) \right|\right)$,
where $g$ is a group action and $h(x)$ is the underlying generating function, affects the accuracy of the subset selected by $k$-NN under the same noisy conditions as in Table~\ref{tab:perm_ortho_comparison}. As shown in Table~\ref{tab:inv_error}, as the error in the representation increases, the performance of the model in identifying an accurate set decreases. 

\begin{table}[htbp!]
    \centering
    \caption{Mean invariance error vs. subset selection accuracy for the Ortho Group.}
    \label{tab:inv_error}
    \resizebox{0.7\linewidth}{!}{%
    \begin{tabular}{cc}
        \toprule
        \textbf{Mean Invariance Error} & \textbf{Accuracy} \\
        \midrule
        \rowcolor{cyan!3}
        $0$     & \textbf{74.31 $\pm$ 1.10} \\
        $0.049$ & $73.86 \pm 1.07$ \\
        $0.111$ & $72.52 \pm 0.71$ \\
        $0.297$ & $68.53 \pm 0.46$ \\
        $0.452$ & $65.06 \pm 1.13$ \\
        \bottomrule
    \end{tabular}
    }
\end{table}
\section{Conclusion}
Training datasets often contain label noise that can significantly affect the accuracy of downstream models. To ensure robustness to such noise, we have shown that knowledge of underlying symmetries and the use of an invariant representation function can guide the selection of an optimal subset of noisy data for downstream classification tasks. Our analysis further demonstrates that even when invariant representations are learned from data, they can still be effective in mitigating noise. Nevertheless, the effectiveness of our method is limited by how accurately group invariance can be modeled and by the extent to which class-relevant information is preserved in the learned representations. As future work, we plan to extend our analysis to additional groups and other data modalities.


\bibliography{uai2026-template}

\newpage

\onecolumn

\title{Leveraging Data Symmetries to Select an Optimal Subset of Training Data under
Label Noise\\(Supplementary Material)}
\maketitle
\appendix
\StartWritingToTOC  
\section*{Appendix Table of  Contents}

{\renewcommand{\baselinestretch}{0.4}\normalsize
\tableofcontents
}
\newpage
\section{Algorithm}
\label{sec:algorithm}
\begin{algorithm}[H]
\caption{Proposed Method}
\label{alg:proposed_method}
\begin{algorithmic}[1]
\REQUIRE Noisy Dataset ($\mathcal{D}_{\text{train}}$), Group ($G$), class label ($\hat{y}$), Percentage of Training Dataset to consider for the downstream model($\tau$). 

\STATE Define an invariant function $\mathcal{F}$  for the group $G$ such that:
\begin{equation*}
\mathcal{F}(x_i) = \mathcal{F}(x_k), \quad \forall x_k \in {g \cdot x_i \mid g \in G}, \quad \text{where } x \in \mathcal{D}_{\text{train}}.
\end{equation*}

\FOR{$i \in [1, \dots, |\mathcal{D}_{\text{train}}|]$}
\STATE Calculate the feature representation $\mathcal{F}(x_i)$ for sample $x_i$.
\ENDFOR

\FOR{$i \in [1, \dots, |\mathcal{D}{\text{train}}|]$}
\FOR{$j \in [1, \dots, |\mathcal{D}{\text{train}}|]$}
\STATE Compute the weight $w_{i,j}$ between samples $x_i$ and $x_j$.
\ENDFOR
\ENDFOR

\FOR{$\hat{y}_i \in \hat{y}$}
\STATE Calculate the class probability $P(\hat{y}_i)$ .
\ENDFOR

\FOR{$i \in [1, \dots, |\mathcal{D}_{\text{train}}|]$}
\STATE Compute the $z_i$ score for sample $x_i$ 
\ENDFOR

\STATE Identify the top $\tau\%$ samples with the smallest $z_i$ scores and use them for downstream model.

\end{algorithmic}
\end{algorithm}

\section{Notations}
\label{sec:notation}
Table \ref{tab:notations} provides details of all notations used in this work.
\begin{table*}[htbp]
    \centering
    \caption{Summary of notations used throughout the paper.}
    \renewcommand{\arraystretch}{1.1} 
    \setlength{\tabcolsep}{8pt}       
    \begin{small}
    \resizebox{\textwidth}{!}{%
    \begin{tabular}{ll}
        \toprule
        \textbf{Notation} & \textbf{Description} \\
        \midrule
        \midrule
        \multicolumn{2}{l}{\textbf{General setting}} \\
        $\gD=\{(x_i,y_i)\}_{i=1}^s$ & Noiseless dataset; $x_i\in\R^d$, $y_i\in\{0,1\}$ drawn from $\gP$. \\
        $\gD_{noisy}=\{(x_i,\hat y_i)\}$ & Noisy-labeled dataset with labels $\hat y_i\in\{0,1\}$. \\
        $\gS\subseteq\gD_{noisy}$ & Subset selected by \textit{cutstats} for training. \\
        $x\in\R^d$ & Input feature; $d$ is feature dimension. \\
        $y,\hat y\in\{0,1\}$ & True and noisy labels. \\
        $\vf:\R^d\to\{0,1\}$ & True labeling function with $y=\vf(x)$. \\
        $\gP$ & Underlying data distribution over $(x,y)$. \\
        $s,d$ & Dataset size and feature dimension. \\
        \midrule
        \multicolumn{2}{l}{\textbf{Groups, invariance, and representations}} \\
        $G$ & Symmetry group acting on inputs; labels invariant to $G$. \\
        $g\cdot x$ & Group action of $g\in G$ on $x$. \\
        $\gO_x=\{g\cdot x:\ g\in G\}$ & Orbit of $x$ under $G$. \\
        $\gF:\R^d\to\R^m$ & Invariant representation; $\gF(x)=\gF(x')$ if $x'\in\gO_x$. \\
        $\gI(y;\gF(x))$ & Mutual information between labels and representation. \\
        $SO(d),\ S_d$ & Special orthogonal (rotations) and symmetric (permutations) groups. \\
        $\hat{\gF}$ & Learned (approx.) invariant representation. \\
        \midrule
        \multicolumn{2}{l}{\textbf{\textit{Cutstats} and $k$-NN machinery}} \\
        $\mathrm{NN}(x)$ & Index set of $k$ nearest neighbors of $x$. \\
        $k$ & Number of neighbors in $k$-NN / \textit{cutstats}. \\
        $w_{i,j}$ & Similarity weight. \\
        $J_i=\sum_{j\in\mathrm{NN}(x_i)} w_{i,j}\,\vone(\hat y_j\neq \hat y_i)$ & Weighted neighbor label disagreement. \\
        $\mu_i,\ \sigma_i^2$ & Mean/variance of $J_i$ under null via $\sP(\hat y)$. \\
        $z_i=\dfrac{J_i-\mu_i}{\sigma_i}$ & \textit{Cutstats} $Z$-score (smaller $\Rightarrow$ cleaner). \\
        $\tau$ & Retained fraction when ranking examples by $z_i$ (top $\tau\%$ kept). \\
        $y^{\text{knn}}$ & Majority label predicted by $k$-NN around a point. \\
        \midrule
        \multicolumn{2}{l}{\textbf{Noise, error rates, and performance}} \\
        $\alpha=\sP(y=0\mid \hat y=1)$,\ \ $\gamma=\sP(y=1\mid \hat y=0)$ & Noise rates. \\
        $\alpha_{noisy},\ \gamma_{noisy}$ & Noise rates on the full noisy dataset (no abstention). \\
        $\alpha_{\gS},\ \gamma_{\gS}$ & Noise rates within the selected subset $\gS$. \\
        $\lambda_0=\sP(y^{\text{knn}}=0\mid y=0)$,\ \ $\lambda_1=\sP(y^{\text{knn}}=1\mid y=1)$ & Class-wise accuracies of $k$-NN. \\
        $\text{err}_{\text{bal}}(\gM,y)$ & Balanced error: $\tfrac12(\sP(\gM(x)=1|y=0)+\sP(\gM(x)=0|y=1))$. \\
        $\phi$ & Abstention label in the subset-selection analysis. \\
        $\gM\in\Gamma,\ \gM^*$ & Downstream classifier and noiseless optimal classifier. \\
        $\mathrm{VC}(\Gamma)$ & Capacity measure (VC dimension) of hypothesis class $\Gamma$. \\
        \midrule
        \multicolumn{2}{l}{\textbf{$k$-NN vs.\ Bayes and constants in bounds}} \\
        $\eta_k(x),\ \eta^*(x)$ & $k$-NN and Bayes-optimal decision rules. \\
        $R_X,\ R^*$ & Risks of $k$-NN and Bayes classifier. \\
        $\nu\in(0,1)$ & Confidence parameter for high-probability statements. \\
        $\beta>0,\ \rho>0$ & Tsybakov (low-noise) and Hölder smoothness parameters. \\
        $K_l(d),K_u(d),K,K'$ & Positive constants depending on $d$/distribution. \\
        $\Delta(\gC)$ & Min.\ separation between corrupted and clean features. \\
        $\lambda$ & Bound term $\displaystyle \sqrt{\frac{\log n+\log(1/\nu)}{k}}+\Big(\frac{k}{n}\Big)^{\rho/d}$. \\
        \bottomrule
    \end{tabular}
    }
    \end{small}
    \label{tab:notations}
\end{table*}




\section{Training Procedure}
\label{sec:train_details}
\textbf{Synthetic Dataset Generation:} For each group, the dataset is generated according to the data generation functions specified in Table~\ref{table:group_data} of the main draft. Each sample ($x_i$) is drawn from a uniform distribution, and its class label ($y$) is determined by thresholding the output of the data-generating function $h(x)$ using $\vone\big(h(x) \geq \theta\big)$. The feature dimension $d$ is set to hundred for Orthogonal groups and five for permutation group. For the Orthogonal group, $h(x)$ is thresholded at 0, while for the Permutation group, the threshold is set based on the mean score. For the Orthogonal group, we set $(c_1, c_2, c_3, k_1, k_2, k_3) = (1, 2, 3, 2, 3, 4)$. For the Permutation group, the hyperparameter $l$ is set to 5. For \textit{cutstats}, we  set $\tau = 40\%$ and used a standard $k$ value of 20 for all experiments, following~\citet{subset_ws_neurips22}. The prior probability for the synthetic experiments isset to $\sP(\hat{y}=1)= \sP(\hat{y}=0) = \frac{1}{2}$.


\textbf{Experiments on Rotated MNIST and Tetris:} 
The experiments on the Rotated MNIST and Tetris datasets are conducted using the standard codebase provided by~\citet{winter2022unsupervised}, with an 80–20\% split between the training and test sets. The unsupervised invariant representation is trained for 20 epochs using the Adam optimizer with a learning rate of $1 \times 10^{-3}$. The same hyperparameters are used consistently for both the group-invariant and contrastive representations. The architecture proposed by~\citet{winter2022unsupervised} is adopted for representation learning in both variants—ours (inv) and ours (contrastive). In cases where the original architecture was unavailable, we constructed an equivalent model with the same number of layers and activation functions. For both datasets, the prior probabilities are set according to the class distributions described in~\citet{subset_ws_neurips22}.

 \textbf{Downstream Model Training:} 
For the synthetic dataset experiments, we use a two-layer neural network with a hidden dimension of 32 units and ReLU activation. The final output layer employs a sigmoid activation function. All models are trained using three different random seeds, a learning rate of $10^{-2}$, and for 20 epochs with a batch size of 1024, optimized using binary cross-entropy loss. 

For the Rotated MNIST dataset, we use a ResNet-9 classifier for downstream training. For the Tetris dataset, we employ a graph neural network (GNN) with three graph convolutional layers for classification.

\textbf{Noise:} The noisy labels (\(\hat{y}\)) are generated by randomly reassigning each sample’s true label to one of the remaining classes with a predefined noise probability, as specified in the main draft. Experiments on the synthetic dataset are conducted with a noise probability of \(0.45\) to simulate highly noisy conditions.

\section{Definitions}
\label{sec:defn}
\subsection {Group and Associated Details}
\begin{definition}[Group]
\label{def:group}
A \emph{group} is a set \( G \) with a binary operation \( \cdot \) that satisfies: 
\begin{itemize}
    \item \textbf{Closure}: \( a, b \in G \Rightarrow a \cdot b \in G \),  
    \item \textbf{Associativity}: \( (a \cdot b) \cdot c = a \cdot (b \cdot c) \),  
    \item \textbf{Identity}: \( \exists e \in G \) such that \( e \cdot a = a \cdot e = a \) for all \( a \in G \),  
    \item \textbf{Inverse}: \( \forall a \in G, \exists a^{-1} \in G \) such that \( a \cdot a^{-1} = e \).  
\end{itemize}
\end{definition}

\begin{definition}[Group Action]
\label{def:group_action}
A \emph{group action} of a group \( G \) on a set \( X \) is a function \( \cdot: G \times X \to X \) satisfying:
\begin{itemize}
    \item[\textbullet] \textbf{Identity}: \( e \cdot x = x \) for all \( x \in X \), where \( e \) is the identity of \( G \).
    \item[\textbullet] \textbf{Compatibility}: \( (gh) \cdot x = g \cdot (h \cdot x) \) for all \( g, h \in G \) and \( x \in X \).
\end{itemize}
\end{definition}

\begin{definition}[Orbit]
\label{def:orbit}
Let \( G \) be a group acting on a set \( X \). The \emph{orbit} of an element \( x \in X \) under the action of \( G \) is the set:
\[
\mathcal{O}_x = \{ g \cdot x \mid g \in G \}.
\]
It consists of all elements in \( X \) that can be reached from \( x \) by applying elements of \( G \).
\end{definition}

\begin{definition}[Invariant Representation Function]
\label{def:inv-rep}
Let \( G \) be a group acting on a set \( X \), and let $\mathcal{O}_x$  denote the orbit of \(x \in X\) under this action (see Definition~\ref{def:orbit}).  
An \emph{invariant representation function} is a mapping
\[
\gF : X \to \mathbb{R}^d
\]
such that
\[
\gF(x) = \gF(x') \quad \text{whenever } x' \in \mathcal{O}_x.
\]
That is, \(\gF\) is constant on each orbit.
\end{definition}
\subsection{$k$-NN bound}
\begin{definition}[Balanced Error]~\cite{subset_ws_neurips22}
\label{def:balance_error}
The balanced error of a classifier \(f\) on labels \(y\) is defined as:
\begin{equation}
\begin{aligned}
    \text{err}_{\text{bal}}(f, y) &= \frac{1}{2} \big(\sP[f(x) = 0 \mid y = 1] \\
    &+ \sP[f(x) = 1 \mid y = 0]\big).
\end{aligned}
\end{equation}
\end{definition}

\begin{definition}[Error Probabilities]
\label{defn:error_prob}
The probability of error in class labels in the dataset w.r.t to the true class label $y  \in \{ 0, 1 \}$ is defined as follows: 
\begin{align}
    \alpha  &= \sP[y = 0 \mid \hat{y} = 1 ], \quad
    \gamma= \sP[y = 1 \mid \hat{y}= 0] \nonumber \\
\end{align}    
\end{definition}

\begin{definition}[Accuracy of K-NN]
\label{defn:acc_heuristic}
The accuracy of K-NN ($\gH : x\rightarrow y^{knn}$) used for subset selection for true class label $y  \in \{ 0, 1 \}$  is defined as follows: 
\begin{align}
    \lambda_0 &= \sP[y^{knn} = 0 \mid y = 0 ], \quad
    \lambda_1 = \sP[y^{knn} = 1 \mid y = 1] 
\end{align}    
\end{definition}
\subsection{Bayes optimal classifier and K-NN}


\begin{definition}[Minimum pairwise distance]
\label{def:min_pairwise}
For corrupted samples $\gC$, the minimal distance between features of any two samples is defined as follows: 
\[
\Delta(\gC) := \min_{x, x' \in \gC, x \neq x'} |x - x'|.
\]
\end{definition}



\section{Assumptions}
\label{sec:assumption}
\subsection{Performance of a Downstream Classifer~\cite{subset_ws_neurips22}}
In this section, we list all the assumptions made by~\citet{subset_ws_neurips22} to derive the bound relating the performance of a classifier trained on a noiseless dataset to that of the same classifier when trained on a subset of a noisy dataset. 

We assume that the nearest neighbors of size $k$ are considered in \textit{cutstats}, with a dataset having equal noise priors $\sP(\hat{y}=0) = \frac{1}{2}$ for a binary classification problem. The selection of a given sample is determined by the label shared among its neighbors.

\begin{assumption}[Nearest Neighbor]
\label{assum:NN}
Let $\mathcal{D} = \{(x_i, \hat{y}_i)\}_{i=1}^n$ denote a dataset used in \textit{cutstats}, where $\hat{y}_i \in \{0, 1\}$ represents the noisy class label and $y_i$ the true class labels. We assume that class priors are balanced, i.e.,
\[
\sP(\hat{y} = 0) = \sP(\hat{y} = 1) = \tfrac{1}{2}. 
\]
\[
\sP(y = 0) = \sP(y = 1) = \tfrac{1}{2}. 
\]
For each sample $x_i$, \textit{cutstats} considers its $k$-nearest neighbors $\mathcal{N}_k(x_i)$ according to a predefined distance metric. The selection mechanism assigns smaller $z$-values to samples whose majority of neighbors share the same label, i.e.,
\[
z_i \downarrow \text{ as } \frac{1}{k} \sum_{x_j \in \mathcal{N}_k(x_i)} \mathbb{I}[\hat{y}_j = \hat{y}_i] \uparrow.
\]
Our analysis therefore, focuses on the statistical behavior of these $k$-nearest neighbors and their corresponding labels under this balanced prior setting.
\end{assumption}

Next, we assume that the predictions made by the underlying heuristic and the noisy labels are conditionally independent given the true class label of an input sample~\citep{subset_ws_neurips22,ratner2016data}.
  
\begin{assumption}[Conditional Independence]
\label{assumption:cond_ind}
The original work of~\citet{subset_ws_neurips22} assumes that the features used to train the model and those used to predict noisy labels ($\hat{y}$) are conditionally independent given the true class label ($y$). Inspired by this, we assume that the prediction made by the $k$-NN classifier and the noisy class label assigned to a given example are conditionally independent given the (unobserved) true label $y$. Intuitively, this assumption ensures that the features used to generate noisy labels (e.g., via crowdsourcing platforms or other heuristic methods) are distinct from the features used for subset selection. Formally,
\begin{equation}
\begin{aligned}
  \sP (\hat{y}, y^{\text{knn}} \mid y) 
  &= \sP (\hat{y} \mid y) \times 
  \sP (y^{\text{knn}} \mid y ).
\end{aligned}
\label{eq:A1}
\end{equation}
\end{assumption}

\subsection{Bayes Optimal Classifier and $k$-NN ~\cite{bahri2020deep}}
We next describe some basic assumptions adopted in previous works~\citep{bahri2020deep, singh2009adaptive} that form the foundation for our analysis of the $k$-NN and Bayes optimal classifiers. 

\begin{assumption}[Support Regularity]
\label{assum:regularity}
There exist constants \( \omega > 0 \) and \( r_0 > 0 \) such that
\[
\text{Vol}(\mathcal{X} \cap B(x,r)) \geq \omega \cdot \text{Vol}(B(x,r))
\]
for all \( x \in \mathcal{X} \) and \( 0 < r < r_0 \), where \( B(x,r) := \{ x' \in \mathcal{X} : \|x - x'\| \leq r \} \).
\end{assumption}

This assumption ensures that the input feature space $\mathcal{X}$ associated with the feature distribution $\mathcal{P}$ does not become arbitrarily thin in any region, thereby guaranteeing sufficient local density for neighborhood-based methods such as $k$-NN.

If \( p_{X} \) is the probability density function associated with  ($\gP$), then the next assumption ensures that for sufficiently large samples, there is a good coverage of the feature space. \cite{bahri2020deep} 
\begin{assumption}(\( p_X \) bounded from below).
\label{assum:bound_prob}
\[
p_{X,0} := \inf_{x \in \mathcal{X}} p_{X}(x) > 0.
\]    
\end{assumption}

Finally, we assume the smoothness of $\eta(x)$~\cite{bahri2020deep,chaudhuri2014rates,reeve2019fast}.
\begin{assumption}(\( \eta \) Hölder continuous)
\label{assum:holder}
 \textit{There exists} \( 0 < \rho \leq 1 \) \textit{and} \( C_{\rho} > 0 \) \textit{such that for class conditional distribution } \big($\eta(x) = \sP(y=1 \mid X = x)$\big) 
\[
|\eta(x) - \eta(x')| \leq C_{\alpha} |x - x'|^{\rho} \quad \text{for all } x, x' \in \mathcal{X}.
\]    
\end{assumption}

\section{Theoretical Analysis}
\label{sec:theoretical_analysis}
\subsection{Performance of a $k$-NN}
\label{proof:acc_heuristic}

\begin{proposition}[Performance of $k$-NN]
\label{alg:performance}
Consider the setup of \citet{subset_ws_neurips22}. 
 For $k$-NN ($\gH : x \mapsto y^{\text{knn}}$), with $y^{knn}$ as the class label of the majority of the samples in the nearest neighbor and
\[
\lambda_0 = \sP(y^{\text{knn}} = 0 \mid y=0), 
\qquad 
\lambda_1 = \sP(y^{\text{knn}} = 1 \mid y=1)
\] 
denote the class-wise accuracies of $k$-NN for class $0$ and class $1$, respectively.  
Then the error rate of the subset selected by $k$-NN satisfies:
\begin{align}
    \alpha_{\gS} &= \frac{\alpha_{noisy} (1 - \lambda_0)}{\alpha_{noisy} (1 - \lambda_0) + (1 - \alpha_{noisy})\lambda_1}, \nonumber \\ 
    \gamma_{\gS} &= \frac{\gamma_{noisy} (1 - \lambda_1)}{\gamma_{noisy} (1 - \lambda_1) + (1 - \gamma_{noisy})\lambda_0} \nonumber
\end{align}
\end{proposition}

\noindent 
\begin{proof}
As per \textit{cutstats}, for any subset selected using $k$-NN, a data point is likely to be included in the subset $\mathcal{S}$ if the label assigned to the sample matches the majority label among its nearest neighbors ($y^{\text{knn}}$), as stated in Assumption~\ref{assum:NN}.

\begin{align}
    \alpha_{\mathcal{S}} 
    &= \sP[y = 0 \mid \hat{y}_{\mathcal{S}} = 1], 
    \quad \text{where $\hat{y}_{\mathcal{S}} = 1$ denotes samples in subset $\mathcal{S}$ ($\hat{y}=1, y^{\text{knn}}=1$)} \\
    &= \frac{\sP(\hat{y} = 1, y^{\text{knn}} = 1  \mid y = 0) \cdot \sP(y = 0)}
    {\sum_{c \in \{0, 1\}} \sP(\hat{y} = 1, y^{\text{knn}} = 1 \mid y = c) \cdot \sP(y = c)} 
    \quad \text{(by Bayes' rule).}
    \label{proof:thm1_eqn1}
\end{align}

\noindent Consider:
\begin{align}
    \sP(\hat{y} = 1, y^{\text{knn}} = 1 \mid y = 0)
    &= \sP(\hat{y} = 1 \mid y = 0) \cdot \sP(y^{\text{knn}} = 1 \mid y = 0) 
    \quad \text{(by Assumption~\ref{assumption:cond_ind})} \nonumber \\
    &= \frac{\sP(y = 0 \mid \hat{y} = 1) \cdot \sP(\hat{y} = 1)}{\sP(y = 0)} \cdot (1 - \lambda_0) \nonumber \\
    &= \frac{\alpha_{noisy} \cdot \sP(\hat{y} = 1) \cdot (1 - \lambda_0)}{\sP(y = 0)}, 
    \quad \text{where $\alpha_{noisy} = \sP(y = 0 \mid \hat{y} = 1)$ for non-abstained data.}
    \label{proof:thm1_eqn2}
\end{align}

\noindent Similarly,
\begin{align}
    \sP(\hat{y} = 1, y^{\text{knn}} = 1 \mid y = 1)
    &= \sP(\hat{y} = 1 \mid y = 1) \cdot \sP(y^{\text{knn}} = 1 \mid y = 1) 
    \quad \text{(by Assumption~\ref{assumption:cond_ind})} \nonumber \\
    &= \frac{\sP(y = 1 \mid \hat{y} = 1) \cdot \sP(\hat{y} = 1)}{\sP(y = 1)} \cdot \lambda_1 \nonumber \\
    &= \frac{(1 - \alpha_{noisy}) \cdot \sP(\hat{y} = 1) \cdot \lambda_1}{\sP(y = 1)}.
    \label{proof:thm1_eqn3}
\end{align}

\noindent Substituting Equations~\ref{proof:thm1_eqn2} and~\ref{proof:thm1_eqn3} into Equation~\ref{proof:thm1_eqn1}, and assuming equal class priors, i.e., $\sP(y = 0) = \sP(y = 1)$, we obtain:
\begin{align*}
    \alpha_{\mathcal{S}} 
    &= \frac{\alpha_{noisy} (1 - \lambda_0)}
    {\alpha_{noisy} (1 - \lambda_0) + (1 - \alpha_{noisy}) \lambda_1}.
\end{align*}

\noindent Similarly, for the other class:
\begin{align}
    \gamma_{\mathcal{S}} 
    &= \frac{\gamma_{noisy} (1 - \lambda_1)}
    {\gamma_{noisy} (1 - \lambda_1) + (1 - \gamma_{noisy}) \lambda_0}.
\end{align}
\end{proof}

\begin{corollary}
\label{corr:acc_performance}
Suppose the class priors are balanced, i.e., $\sP(y=0) = \sP(y=1) = \tfrac{1}{2}$. If the $k$-NN classifier has error $\sP(y^{\text{knn}} \neq y) \leq \tfrac{1}{2}$, then the subset error rates satisfy
\[
    \alpha_{\gS} \leq \alpha_{noisy}
    \quad \text{and} \quad
    \gamma_{\gS} \leq \gamma_{noisy}.
\]
  Moreover, as $\lambda_{0}, \lambda_{1} \to 1$, the subset error rates also reduces:
\[
    \alpha_{\gS}, \, \gamma_{\gS} \;\to\; 0.
\]
\end{corollary}

\begin{proof}
 By Proposition~\ref{alg:performance},
\[
\alpha_{\gS}
=\frac{\alpha_{noisy}(1-\lambda_0)}{\alpha_{noisy}(1-\lambda_0)+(1-\alpha_{noisy})\lambda_1},
\qquad
\gamma_{\gS}
=\frac{\gamma_{noisy}(1-\lambda_1)}{\gamma_{noisy}(1-\lambda_1)+(1-\gamma_{noisy})\lambda_0}.
\]

With balanced priors, the overall $k$-NN accuracy is
\[
\sP(y^{\text{knn}}=y)=\tfrac12\lambda_0+\tfrac12\lambda_1=\tfrac12(\lambda_0+\lambda_1).
\]
The assumption $\sP(y^{\text{knn}}\neq y)\le \tfrac12$ is equivalent to
$\sP(y^{\text{knn}}=y)\ge \tfrac12$, hence
\[
\lambda_0+\lambda_1\;\ge\;1
\quad\Longleftrightarrow\quad
1-\lambda_0\;\le\;\lambda_1
\quad\text{and}\quad
1-\lambda_1\;\le\;\lambda_0.
\]

Consider the function $f(a)=\dfrac{au}{au+(1-a)v}$ for $a\in(0,1)$ and $u,v>0$.
A direct algebraic check shows
\[
f(a)\le a
\;\Longleftrightarrow\;
a u \le a^2 u + a(1-a)v
\;\Longleftrightarrow\;
a(1-a)(u-v)\le 0
\;\Longleftrightarrow\;
u\le v.
\]
Apply this with $a=\alpha_{noisy}$, $u=1-\lambda_0$, $v=\lambda_1$. Since $1-\lambda_0\le\lambda_1$, we get $\alpha_{\gS}\le \alpha_{noisy}$.  
Likewise, with $a=\gamma_{noisy}$, $u=1-\lambda_1$, $v=\lambda_0$ and $1-\lambda_1\le\lambda_0$, we get $\gamma_{\gS}\le \gamma_{noisy}$.

For the limit claim, let $\lambda_0,\lambda_1\to 1$. Then $1-\lambda_0\to 0$ and $1-\lambda_1\to 0$, so from the formulas above,
\[
\alpha_{\gS}\;=\;\frac{\alpha_{noisy}(1-\lambda_0)}{\alpha_{noisy}(1-\lambda_0)+(1-\alpha_{noisy})\lambda_1}\;\longrightarrow\;0,
\qquad
\gamma_{\gS}\;\longrightarrow\;0.
\]   
\end{proof}
\subsection{Bayes Optimal Classifier and k-NN}

In this section, we show how leveraging underlying invariances can enable the $k$-NN classifier to perform comparably to the Bayes optimal classifier. Before proceeding, we establish two important lemmas that form the foundation of our proof. The first lemma demonstrates that the volume of a unit ball decreases with increasing feature dimension $d$.
\subsubsection{Lemma}
\begin{lemma}[Super-exponential decay of the unit ball volume~\cite{wiki:n_ball_volume,vershynin2018high}]
\label{lemma:unitball}
Let \( V_d \) denote the volume of the unit-radius Euclidean ball in \( \mathbb{R}^d \):
\[
V_d = \frac{\pi^{d/2}}{\Gamma\!\left(\frac{d}{2} + 1\right)}.
\]
where, $\Gamma$ denotes gamma function. Then \( V_d \to 0 \) as \( d \to \infty \), and moreover \( V_d \) decays super-exponentially, that is,
\[
\forall\, c > 1 \;\; \exists\, D \text{ such that } d \ge D \implies V_d \le c^{-d}.
\]
\end{lemma}

\begin{proof}
Using the standard Stirling lower bound for the Gamma function,
\[
\Gamma(x+1) \ge \sqrt{2 \pi x} \left( \frac{x}{e} \right)^{x}, \qquad x > 0.
\]
Setting \( x = \frac{d}{2} \), we have
\[
\Gamma\!\left( \frac{d}{2} + 1 \right)
    \ge \sqrt{\pi d} \left( \frac{d}{2e} \right)^{\frac{d}{2}}.
\]
Substituting this into the formula for \( V_d \) gives
\[
V_d
    = \frac{\pi^{d/2}}{\Gamma\!\left( \frac{d}{2} + 1 \right)}
    \le \frac{\pi^{d/2}}{\sqrt{\pi d} \left( \frac{d}{2e} \right)^{d/2}}
    = \frac{1}{\sqrt{\pi d}} \left( \frac{2 \pi e}{d} \right)^{\! d/2}.
\]
Taking logarithms,
\[
\log V_d
    \le -\frac{1}{2} \log(\pi d)
         + \frac{d}{2} \left( \log(2 \pi e) - \log d \right)
    = -\frac{d}{2} \log d + O(d).
\]
Therefore,
\[
V_d \le \exp\!\big( -\tfrac{1}{2} d \log d + O(d) \big).
\]
For any fixed \( c > 1 \), since \( \tfrac{1}{2} \log d \to \infty \), we can find \( D \) such that for all \( d \ge D \),
\[
\tfrac{1}{2} \log d - C \ge \log c
\]
for some constant \( C \) absorbing the \( O(1) \) term. Hence for all \( d \ge D \),
\[
\log V_d \le -d \log c,
\]
i.e.\ \( V_d \le c^{-d} \). Thus, \( V_d \) decays faster than any exponential \( c^{-d} \), proving the super-exponential rate.

\end{proof}

The next lemma formally establishes how, for data exhibiting permutation invariance, the sorting function influences the induced density.

\begin{lemma}
\label{lemma:permut_sort}
Let $X=(X_1,\dots,X_d)\in\mathbb{R}^d$ be an absolutely continuous random vector with density $p_X$.
Assume $p_X$ is \emph{permutation--invariant}, i.e.
\[
  p_X(x_1,\dots,x_d)=p_X\bigl(x_{\pi(1)},\dots,x_{\pi(d)}\bigr)
  \quad\text{for every permutation }\pi\in S_d .
\]
Let $\operatorname{sort}(x)$ be the non-decreasing rearrangement of $x=(x_1,\dots,x_d)$, and set $Y=\operatorname{sort}(X)$.
Then for every $y$ with $y_1<y_2<\dots<y_d$,
\[
  p_Y(y)=d!\,p_X(y).
\]
\end{lemma}

\begin{proof}
Let $C=\{y\in\mathbb{R}^d:\ y_1<\cdots<y_d\}$.
For any $y\in C$, the preimage of $y$ under the sorting map consists of the $d!$ permutations,
\[
  \operatorname{sort}^{-1}(y)=\{\pi(y):\ \pi\in S_d\}.
\]
More generally, for measurable $A\subseteq C$ we have the (essentially disjoint) decomposition
\[
  \operatorname{sort}^{-1}(A)=\bigsqcup_{\pi\in S_d} \pi(A),
\]

Therefore,
\[
  \Pr(Y\in A)
  =\Pr\bigl(X\in \operatorname{sort}^{-1}(A)\bigr)
  =\sum_{\pi\in S_d} \Pr\bigl(X\in \pi(A)\bigr)
  =\sum_{\pi\in S_d} \int_{\pi(A)} p_X(x)\,dx .
\]
Each permutation $\pi:\mathbb{R}^d\to\mathbb{R}^d$ is a linear isometry with Jacobian determinant $\pm1$, hence it is measure-preserving. By change of variables,
\[
  \int_{\pi(A)} p_X(x)\,dx \;=\; \int_A p_X\bigl(\pi(y)\bigr)\,dy .
\]
Using permutation invariance, $p_X(\pi(y))=p_X(y)$ for all $\pi\in S_d$, so
\[
  \Pr(Y\in A)=\sum_{\pi\in S_d}\int_A p_X(y)\,dy
  \;=\; d!\int_A p_X(y)\,dy .
\]
Since this holds for every measurable $A\subseteq C$, the integrand identifies the density of $Y$ on $C$:
\[
  p_Y(y)=d!\,p_X(y)\qquad\text{for a.e. }y\in C.
\]
Finally, the boundary $\{y_i=y_j\text{ for some }i<j\}$ has Lebesgue measure zero, so the same equality holds almost everywhere on $\{y_1\le\cdots\le y_d\}$.
\end{proof}
\subsubsection{Analysis for bayes classifier}
Since our proof builds upon the theoretical framework of~\citet{bahri2020deep}, we first restate their main result for completeness. According to their analysis, if a dataset contains $n$ uncorrupted samples and a subset $\mathcal{C}$ of samples with corrupted (noisy) labels, then for a $k$-nearest neighbors classifier $\eta_k(x)$ and the Bayes-optimal classifier $\eta^*(x)$, the following relationship holds:

\begin{lemma}\citep{bahri2020deep}
\label{thm:rate_tsyb_original}
Let $\nu > 0$, and assume the conditions of \citet{bahri2020deep} hold under the Tsybakov noise condition with parameter $\beta$. Then there exist constants $K_l(d), K_u(d), K, K' > 0$, depending only on the dimension $d$ and the underlying distribution, such that with probability at least $1 - \nu$ :

 If \( k \) lies in the range


\[
K_l(d) \cdot \log^2(1/\nu) \cdot n^{\frac{\rho}{\rho+d}} \leq k \leq K_u(d) \cdot \Delta(\mathcal{C})^d \cdot n,
\] Then,
\[
\mathbb{P} \big(\eta_k(x) \neq \eta^*(x)\big) \leq K \cdot \lambda^\beta,
\]
\[
R_X - R^* \leq K' \cdot \lambda^{\beta+1},
\]
where
\[
\lambda = \left( \sqrt{\frac{\log n + \log(1/\nu)}{k}} + \left( \frac{k}{n} \right)^{\rho/d} \right),
\]

  $R_X$ and $R^*$ denote the risks of the $k$-NN and the Bayes optimal classifier, respectively.
\end{lemma}
where \(\Delta(\mathcal{C})\) denotes the minimum distance between features of corrupted examples (Definition~\ref{def:min_pairwise}) and $\rho$ is a parameter associated with Holder's continuity (Assumption~\ref{assum:holder}).

\subsection{Impossibility Condition and Group Analysis }
\label{appx_proof_bayes}
\begin{proposition}[High-dimensional $k$-NN and Bayes Classifier ]
\label{prop:impossible_condn}
As the feature dimension $d$ increases, the factor $K_u(d) \cdot \Delta(\mathcal{C})^d \cdot n$ in Lemma~\ref{thm:rate_tsyb_original} decreases with $d$ \big($K_u(d)$  decreases super exponentially with $d$ \big). Consequently, there exists a threshold $d_0$ such that for all data dimension $d \geq d_0$ and for a  given n, $\nu$, $\rho$ and $\Delta$ no $k$ satisfies
\[
    K_l(d) \cdot \log^2(1/\nu) \cdot n^{\frac{\rho}{\rho+d}} \leq k \leq K_u(d) \cdot \Delta(\mathcal{C})^d \cdot n.
\]
\end{proposition}
\begin{proof}
According to \citet{bahri2020deep_supp}, we have 
\[
    K_u(d) = \omega \cdot v_d \cdot p_{x,0},
\]
Where $v_d$ denotes the volume of the unit ball in $\mathbb{R}^d$, $p_{x,0}$ is the lower bound on the data density (as defined in Assumption~\ref{assum:bound_prob}), and $\omega$ is the regularity constant from Assumption~\ref{assum:regularity}.

$K_l(d)$ grows approximately linearly with $d$~\cite{bahri2020deep_supp}. However, the unit ball volume $v_d$ decreases faster than exponentially as $d$ increases (see Lemma~\ref{lemma:unitball}). Therefore, $K_u(d)$ decays super-exponentially with $d$, causing the upper bound 
\[
    K_u(d) \cdot \Delta(\mathcal{C})^d \cdot n
\]
to shrink rapidly with increasing dimension.

Hence, there exists a threshold $d_0$ such that for all $d > d_0$,
\[
    K_u(d) \cdot \Delta(\mathcal{C})^d \cdot n 
    < 
    K_l(d) \cdot \log^2(1/\nu) \cdot n^{\frac{\rho}{\rho + d}},
\]
implying that no valid $k$ can satisfy Lemma~\ref{thm:rate_tsyb_original} the above inequality for the given parameters $n$, $\nu$, $\rho$, and $\Delta$.
\end{proof}
\begin{proposition}[Orthogonal Group]
\label{prop:orthogonal_grp}
For data whose true-labeling function  satisfy orthogonal invariance ($SO(d)$), if $k$-NN utilizes the invariant representation function ($\gF$), and the inequality holds for dimension ($d$=1), i.e., 
\[
    K_l(1) \cdot \log^2(1/\nu) \cdot n^{\frac{\rho}{\rho+1}} \leq k \leq K_u(1) \cdot \Delta(\mathcal{C}) \cdot n.
\]

Then there exists a $k$  for all dimensions $d$ for which $k$-NN will approximate the Bayes optimal classifier for a  given $n$, $\nu$, $\rho$ and $\Delta$,
\end{proposition}
\begin{proof}
Under orthogonal invariance ($SO(d)$), the invariant representation function $\gF$ can be chosen as the Euclidean norm, i.e., $\gF(x) = \|x\|_2$. This mapping reduces any $d$-dimensional input to a one-dimensional scalar that is invariant to orthogonal transformations.

The difficulty discussed in Proposition~\ref{prop:impossible_condn} arises due to the exponential dependence on the ambient dimension $d$. However, by using the invariant function $\gF$, the data are effectively mapped from $\mathbb{R}^d$ to $\mathbb{R}$, eliminating the dimensional dependence responsible for the degradation in $K_u(d)$.

Consequently, if the condition
\[
    K_l(1) \cdot \log^2(1/\nu) \cdot n^{\frac{\rho}{\rho + 1}} 
    \leq 
    k 
    \leq 
    K_u(1) \cdot \Delta(\mathcal{C}) \cdot n
\]
is satisfied in the one-dimensional case, it will also hold for data in any dimension $d$ when using the invariant representation $\gF(x) = \|x\|_2$. Therefore, the $k$-NN classifier based on $\gF$ will approximate the Bayes optimal classifier for all $d$.
\end{proof}
\begin{proposition}
\label{prop:perm_grp}
For data whose true-labeling function satisfies permutation  invariance ($S_d$), if $k$-NN utilizes the invariant representation function ($\gF$), then $K_u(d) \cdot \Delta(\mathcal{C})^d \cdot n$ does not decrease with d, and there will exist $k$ for a  given $n$, $\nu$, $\rho$ and $\Delta$ for which
\[
    K_l(d) \cdot \log^2(1/\nu) \cdot n^{\frac{\rho}{\rho+d}} \leq k \leq K_u(d) \cdot \Delta(\mathcal{C})^d \cdot n.
\] holds, and $k$-NN can approximate Bayes optimal classifier. 
\end{proposition}
\begin{proof}
Under permutation invariance ($S_d$), the invariant representation function $\gF$ can be taken as the sorting operator, which maps any input vector to a canonical ordered form. As shown in Lemma~\ref{lemma:permut_sort}, this operation can increase the effective probability density by a factor of $d!$, since $d!$ samples in the original data space correspond to a single representation in the invariant (sorted) space.

The constant, $K_u(d) = \omega \cdot v_d \cdot p_{x,0}$ (Proposition~\ref{prop:impossible_condn}), where $v_d$ is the volume of the $d$-dimensional unit ball, $\omega$ is a regularity constant and $p_{x,0}$ is the lower bound of the data density. While $v_d$ decreases super-exponentially with $d$, the sorting operation effectively multiplies $p_{x,0}$ by a $d!$ factor, compensating for this decay.

Therefore, the product $K_u(d) \cdot \Delta(\mathcal{C})^d \cdot n$ does not diminish with increasing $d$. Consequently, there exists a valid $k$ satisfying the inequality in Lemma~\ref{thm:rate_tsyb_original}, and the $k$-NN classifier based on the permutation-invariant representation $\gF$ can approximate the Bayes optimal classifier.
\end{proof}
\section{Comparison with Robustness-Based Baseline}
\label{appx:robust}
We have primarily compared our method with subset-selection based baselines, as they are most closely aligned with our approach and are designed to improve the quality of the training dataset. However, for benchmarking purposes , we additionally include a comparison with the Generalized Cross Entropy (GCE) loss~\cite{zhang2018generalized}, a model-robustness-based method. Since GCE does not modify or improve the training data itself, we report only classifier accuracy for this baseline. As shown in Tables~\ref{tab:gce_group}--\ref{tab:gce_tetris}, our method continues to outperform the robustness-based loss function in most settings across all three benchmarks.

\begin{table}[h]
    \centering
    \caption{Classifier accuracy on Orthogonal Group and Permutation Group tasks.}
    \label{tab:gce_group}
    \begin{tabular}{lcc}
        \toprule
        \textbf{Method} & \textbf{Orthogonal Group} & \textbf{Permutation Group} \\
        \midrule
        Complete noisy + GCE & $56.43 \pm 2.23$ & $81.17 \pm 0.024$ \\
        Ours                 & $\mathbf{68.52 \pm 0.66}$ & $\mathbf{86.00 \pm 0.84}$ \\
        \bottomrule
    \end{tabular}
\end{table}

\begin{table}[h]
    \centering
    \caption{Classifier accuracy across noise levels on Rot-MNIST.}
    \label{tab:gce_rotmnist}
    \begin{tabular}{lccc}
        \toprule
        \textbf{Method} & \textbf{Noise 0.4} & \textbf{Noise 0.6} & \textbf{Noise 0.8} \\
        \midrule
        Complete noisy + GCE  & $87.38 \pm 1.96$          & $85.23 \pm 1.03$          & $61.57 \pm 0.64$ \\
        Ours (contrastive)    & $83.95 \pm 2.04$          & $76.28 \pm 0.57$          & $41.52 \pm 3.71$ \\
        Ours (group inv)      & $\mathbf{91.81 \pm 0.58}$ & $\mathbf{89.37 \pm 1.83}$ & $\mathbf{75.73 \pm 5.99}$ \\
        \bottomrule
    \end{tabular}
\end{table}

\begin{table}[h]
    \centering
    \caption{Classifier accuracy across noise levels on Tetris.}
    \label{tab:gce_tetris}
    \begin{tabular}{lccc}
        \toprule
        \textbf{Method} & \textbf{Noise 0.4} & \textbf{Noise 0.6} & \textbf{Noise 0.8} \\
        \midrule
        Complete noisy + GCE  & $\mathbf{88.18 \pm 0.33}$ & $\mathbf{75.65 \pm 1.75}$ & $38.38 \pm 0.63$ \\
        Ours (contrastive)    & $52.12 \pm 1.86$          & $50.90 \pm 2.52$          & $30.18 \pm 4.46$ \\
        Ours (group inv)      & $75.72 \pm 5.65$          & $67.28 \pm 6.33$          & $\mathbf{40.95 \pm 5.34}$ \\
        \bottomrule
    \end{tabular}
\end{table}

\section{Performance of K-NN for High Dimension $d$}
\label{appx:high_dim}
We conducted an additional ablation comparing our method with \textbf{vanilla Cutstats} across different data dimensions for both the orthogonal and permutation groups. Tables~\ref{tab:knn_orthogonal} and~\ref{tab:knn_permutation} report the accuracy of K-NN in recovering the clean subset when using the original features (vanilla Cutstats), and our invariant features.
The results clearly show that as the dimensionality increases, the performance of K-NN with vanilla Cutstats degrades significantly, whereas using invariant representations helps maintain strong recovery performance.

\begin{table}[H]
    \centering
    \caption{Accuracy of K-NN across data dimensions $d$ for the Orthogonal Group.}
    \label{tab:knn_orthogonal}
    \begin{tabular}{lccccc}
        \toprule
        \textbf{Method} & $d = 200$ & $d = 400$ & $d = 600$ & $d = 800$ & $d = 1000$ \\
        \midrule
        Vanilla Cutstats 
            & $73.59 \pm 2.59$ 
            & $70.73 \pm 1.26$ 
            & $65.96 \pm 3.83$ 
            & $62.90 \pm 3.30$ 
            & $55.35 \pm 0.33$ \\
        Ours             
            & $\mathbf{73.97 \pm 2.52}$ 
            & $\mathbf{76.47 \pm 0.99}$ 
            & $\mathbf{75.65 \pm 1.03}$ 
            & $\mathbf{73.30 \pm 2.26}$ 
            & $\mathbf{75.31 \pm 1.83}$ \\
        \bottomrule
    \end{tabular}
\end{table}

\begin{table}[H]
    \centering
    \caption{Accuracy of K-NN across data dimensions $d$ for the Permutation Group.}
    \label{tab:knn_permutation}
    \begin{tabular}{lccccc}
        \toprule
        \textbf{Method} & $d = 200$ & $d = 400$ & $d = 600$ & $d = 800$ & $d = 1000$ \\
        \midrule
        Vanilla Cutstats 
            & $57.08 \pm 0.53$ 
            & $56.57 \pm 0.73$ 
            & $56.15 \pm 0.50$ 
            & $55.84 \pm 1.04$ 
            & $55.91 \pm 0.77$ \\
        Ours             
            & $\mathbf{70.03 \pm 1.19}$ 
            & $\mathbf{72.52 \pm 0.90}$ 
            & $\mathbf{71.27 \pm 1.64}$ 
            & $\mathbf{72.21 \pm 2.11}$ 
            & $\mathbf{72.07 \pm 1.66}$ \\
        \bottomrule
    \end{tabular}
\end{table}

\section{Evaluation on CIFAR-10N}
\label{appx:cifar10}
We evaluated \textbf{Ours (contrastive)}, which does not rely on explicit group knowledge, on the worst-label annotation of \textbf{CIFAR-10N} using the same setup as Table~\ref{tab:perm_ortho_comparison} in the main draft. The corresponding results are reported in Table~\ref{tab:cifar10n}, where our method outperforms all baselines across both classifier accuracy and subset accuracy.

\begin{table}[H]
    \centering
    \caption{Classifier accuracy and subset accuracy on CIFAR-10N (Worst Labels).}
    \label{tab:cifar10n}
    \begin{tabular}{lcc}
        \toprule
        \textbf{Method} & \textbf{Classifier-Acc} & \textbf{Subset-Acc} \\
        \midrule
        NoiseLess            & $82.59 \pm 0.32$          & $100.00 \pm 0.00$ \\
        \midrule
        Noisy-complete       & $58.55 \pm 2.68$          & $59.79 \pm 0.00$  \\
        Random               & $58.78 \pm 0.27$          & $60.02 \pm 0.34$  \\
        Vanilla Cutstats     & $51.42 \pm 1.53$          & $77.25 \pm 0.00$  \\
        Entropy              & $62.17 \pm 2.80$          & $78.15 \pm 0.54$  \\
        Forget               & $50.54 \pm 1.96$          & $51.73 \pm 0.28$  \\
        Herding              & $60.08 \pm 0.41$          & $62.31 \pm 0.25$  \\
        Ours (contrastive)   & $\mathbf{65.30 \pm 1.88}$ & $\mathbf{90.48 \pm 0.25}$ \\
        \bottomrule
    \end{tabular}
\end{table}

\section{Ablation study}
\label{sec:ablation}
Table~\ref{tab:knn_ablation} presents the ablation results on Rotated MNIST for different values of $k$ in $k$-NN with a fixed subset size of 40\%, while Table~\ref{tab:subset_ablation} shows a similar ablation for a fixed $k$ (k=20) value but varying subset sizes.
\begin{table*}[htbp]
\centering
\caption{Ablation study: Performance comparison across different subset sizes for K=20 under varying noise probabilities. The table reports the mean across three different seeds.}
\label{tab:subset_ablation}
\resizebox{\textwidth}{!}{%
\begin{tabular}{@{}llcccccccc@{}}
\toprule
\multirow{2}{*}{\textbf{Noise Prob}} & 
\multirow{2}{*}{\textbf{Method}} & 
\multicolumn{2}{c}{\textbf{20\%}} & 
\multicolumn{2}{c}{\textbf{40\%}} & 
\multicolumn{2}{c}{\textbf{60\%}} & 
\multicolumn{2}{c}{\textbf{80\%}} \\
\cmidrule(lr){3-4} \cmidrule(lr){5-6} \cmidrule(lr){7-8} \cmidrule(lr){9-10}
& & \textbf{Cls-Acc} & \textbf{Sub-Acc} & \textbf{Cls-Acc} & \textbf{Sub-Acc} & \textbf{Cls-Acc} & \textbf{Sub-Acc} & \textbf{Cls-Acc} & \textbf{Sub-Acc} \\
\midrule
\multirow{2}{*}{0.6} & \textbf{Ours (group inv)} & \textbf{87.10} & \textbf{99.34} & \textbf{89.37} & \textbf{90.51} & \textbf{89.73} & \textbf{64.83} & \textbf{80.47} & \textbf{49.57} \\
& Ours (contrastive) & 58.59 & 64.82 & 76.28 & 53.30 & 82.62 & 47.79 & 77.37 & 43.73 \\
\midrule
\multirow{2}{*}{0.8} & \textbf{Ours (group inv)} & \textbf{63.58} & \textbf{53.04} & \textbf{75.73} & \textbf{37.77} & \textbf{69.37} & \textbf{29.53} & \textbf{69.76} & \textbf{24.03} \\
& Ours (contrastive) & 35.51 & 27.06 & 41.52 & 24.24 & 50.24 & 22.39 & 52.49 & 21.21 \\
\bottomrule
\end{tabular}%
}
\end{table*}

\begin{table*}[htbp]
\centering
\caption{Ablation study: Performance comparison across different K-NN values under varying noise probabilities for classifier accuracy (Cls-Acc) and subset accuracy (Sub-Acc). The table reports the mean across three different seeds.}
\label{tab:knn_ablation}
\resizebox{\textwidth}{!}{%
\begin{tabular}{@{}llcccccccccc@{}}
\toprule
\multirow{2}{*}{\textbf{Noise Prob}} & 
\multirow{2}{*}{\textbf{Method}} & 
\multicolumn{2}{c}{\textbf{K=20}} & 
\multicolumn{2}{c}{\textbf{K=40}} & 
\multicolumn{2}{c}{\textbf{K=80}} & 
\multicolumn{2}{c}{\textbf{K=160}} & 
\multicolumn{2}{c}{\textbf{K=320}} \\
\cmidrule(lr){3-4} \cmidrule(lr){5-6} \cmidrule(lr){7-8} \cmidrule(lr){9-10} \cmidrule(lr){11-12}
& & \textbf{Cls-Acc} & \textbf{Sub-Acc} & \textbf{Cls-Acc} & \textbf{Sub-Acc} & \textbf{Cls-Acc} & \textbf{Sub-Acc} & \textbf{Cls-Acc} & \textbf{Sub-Acc} & \textbf{Cls-Acc} & \textbf{Sub-Acc} \\
\midrule
\multirow{2}{*}{0.6} & \textbf{Ours (group inv)} & \textbf{89.37} & \textbf{90.51} & \textbf{91.23} & \textbf{93.40} & \textbf{90.34} & \textbf{94.29} & \textbf{90.82} & \textbf{94.43} & \textbf{87.40} & \textbf{93.91} \\
& Ours (contrastive) & 76.28 & 53.30 & 68.70 & 56.29 & 65.49 & 58.90 & 70.19 & 60.86 & 62.65 & 62.14 \\
\midrule
\multirow{2}{*}{0.8} & \textbf{Ours (group inv)} & \textbf{75.73} & \textbf{37.77} & \textbf{76.21} & \textbf{42.57} & \textbf{80.95} & \textbf{45.92} & \textbf{76.33} & \textbf{47.44} & \textbf{76.30} & \textbf{48.07} \\
& Ours (contrastive) & 41.52 & 24.24 & 43.23 & 25.47 & 49.79 & 26.85 & 52.00 & 28.49 & 46.13 & 30.15 \\
\bottomrule
\end{tabular}%
}
\end{table*}

\end{document}